\documentclass[runningheads]{llncs}

\pdfoutput=1

\usepackage{scrextend}

\usepackage[T1]{fontenc}
\usepackage{graphicx}
\usepackage{microtype}

\usepackage[hidelinks]{hyperref}
\usepackage{float}

\usepackage{color}

\spnewtheorem{observation}[theorem]{Observation}{\bfseries}{\itshape}

\newcommand{\cclass} [1] {\textnormal{\textsf{#1}}}

\input{styles/include.sty}
\input{styles/problems.sty}
\input{styles/variables.sty}
\input{styles/utilities.sty}

\begin{document}
\title{Leveraging Fixed-Parameter Tractability for Robot Inspection Planning}
\author{
    Yosuke Mizutani\inst{1} \and
    Daniel Coimbra Salomao\inst{1} \and
    Alex Crane\inst{1} \and
    Matthias Bentert\inst{2} \and
    Pål Grønås Drange\inst{2} \and
    Felix Reidl\inst{3} \and
    Alan Kuntz\inst{1} \and
    Blair D. Sullivan\inst{1}
}

\authorrunning{Y.\ Mizutani et al.}
\institute{
    University of Utah, USA \and
    University of Bergen, Norway \and
    Birkbeck, University of London, UK}
\maketitle

\begin{abstract}\label{abstract}
    Autonomous robotic inspection, where a robot moves through its environment and inspects points of interest, has applications in industrial settings, structural health monitoring, and medicine.
    Planning the paths for a robot to safely and efficiently perform such an inspection is an extremely difficult algorithmic challenge.
    In this work we consider an abstraction of the inspection planning problem which we term \gi.
    We give two exact algorithms for this problem, using dynamic programming and integer linear programming.  We analyze the performance of these methods, and present multiple approaches to achieve scalability. 
    We demonstrate significant improvement both in path weight and inspection coverage over a state-of-the-art approach on two robotics tasks in simulation, a bridge inspection task by a UAV and a surgical inspection task using a medical robot.\looseness-1

\end{abstract}

\section{Introduction}\label{sec:intro}

Inspection planning, where a robot is tasked with planning a path through its environment
to sense a set of points of interest (POIs) has broad potential applications.
These include the inspection of surfaces to identify defects in industrial settings such as
car surfaces~\cite{atkar2005uniform},
urban structures~\cite{cheng2008time}, and
marine vessels~\cite{englot2017planning},
as well as in medical applications to enable the mapping of
subsurface anatomy~\cite{Cho2021_ISMR,Cho2024_ICRA} or disease diagnosis.

Consider the demonstrative medical example of diagnosing the cause of pleural effusion,
a medical condition in which a patient's pleural space---the area between the lung and the chest wall---%
fills with fluid, collapsing the patient's lung~\cite{Light2007_Pleural,Harris1995_Chest,Noppen2010_SRCCM}.
Pleural effusion is a symptom, albeit a serious one, of one of over fifty underlying causes,
and the treatment plan varies significantly,
depending heavily on which of the underlying conditions has caused the effusion.
To diagnose the underlying cause, physicians will insert an endoscope into the pleural space
and attempt to inspect areas of the patient's lung and chest wall.
Automated medical robots have been proposed as a potential assistive technology with great promise
to ease the burden of this difficult diagnostic procedure.
However, planning the motions to inspect the inside of a patient's body with a medical robot,
or indeed \emph{any environment with any robot} is an extremely challenging problem.

Planning a motion for a robot to move from a single configuration to another configuration is,
under reasonable assumptions, known to be \cclass{PSPACE}-hard~\cite{Halperin2017_Book}.
Inspection planning extends this typical motion planning problem
by requiring the traversal of multiple configurations.
The planned route may need to include complexities such as
tracing back to where the robot has already been and/or traversing circuitous routes.
Consider~\Cref{fig:intro}, where examples are given of inspections that necessitate circuitous paths or backtracking during inspection.
While the specific examples given are intuitive in the robots' workspaces, cycles and backtracking may be required in the c-space graph in ways that don't manifest intuitively in the workspace as well.

\begin{figure}[t]
    \centering
    \includegraphics[width=0.75\columnwidth]{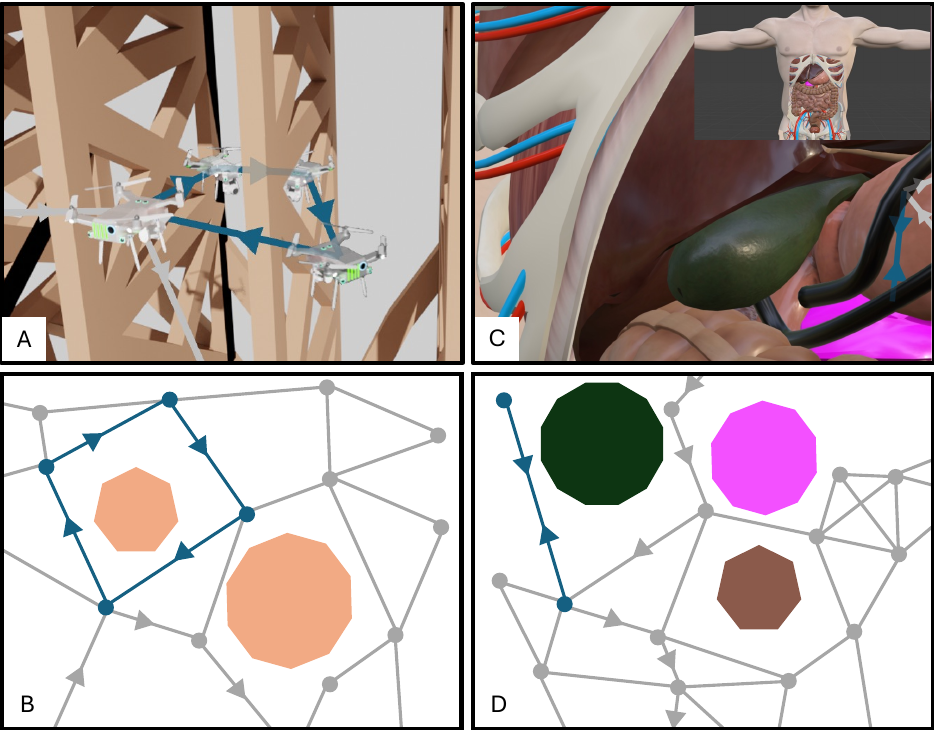}
    \caption{Inspection planning, in contrast to traditional motion planning, may necessitate leveraging cycles and backtracking on graphs embedded in the robot's configuration space. This necessitates computing a walk (rather than a path) on a graph. (A) A quadrotor, while inspecting a bridge for potential structural defects, may need to circle around obstacles, (B) leveraging a cycle in its c-space graph (teal). (C) A medical endoscopic robot (black) may need to move into and then out of an anatomical cavity to, e.g., visualize the underside of a patient's gallbladder (green), (D) requiring backtracking in its c-space graph (teal).}
    \label{fig:intro}
\end{figure}

Further, it is almost certainly not sufficient to only consider the ability to inspect the POIs,
but one must also consider the \emph{cost} of the path taken to inspect them.
This is because while inspecting POIs may be an important objective,
it is not the only objective in real robotics considerations;
unmanned aerial vehicles (UAVs) must operate within their battery capabilities,
and medical robots must consider the time a patient is subjected to a given procedure.

The state-of-the-art in inspection planning, presented by Fu \emph{et al.}~\cite{fu2021computationally} and named \iriscli,
casts this problem as an iterative process with two phases.
In the first phase, a rapidly-exploring random graph (RRG)~\cite{karaman2011sampling} is constructed.
In this graph, each vertex represents a possible configuration of the robot, edges indicate the ability to
transition between configuration states, and edge weights indicate the cost of these transitions. Additionally, every
vertex is labeled with the set of POIs which may be inspected by the robot when in the associated configuration state.
In the second phase, a walk is computed in this graph with the dual objectives of (a) inspecting all POIs and (b) minimizing the
total weight (cost) of the walk. By repeating both phases iteratively, Fu \emph{et al.} are able to guarantee
asymptotic optimality\footnote{See~\cite{fu2021computationally} for specific definition of asymptotic optimality in their case.} of the resulting inspection plan.

In this work we focus on improving the second phase.
We formulate this phase as an algorithmic problem on edge-weighted and vertex-multicolored graphs,
which we call \gi{} (formally defined in~\Cref{sec:prelims}).
\gi{} is a generalization of the well-studied \TSP{} (\TSPshort) problem\footnote{
  Given an edge-weighted graph and a start-vertex $s$,
  compute a minimum-weight closed walk from $s$ visiting every vertex exactly once.}~\cite{applegate2006traveling}. As such, it is deeply related to the rich literature on
``color collecting'' problems studied by the graph algorithms community.

\gi{} is closely related to the \textsc{Generalized Traveling Salesperson Problem} (also known as \prob{Group TSP}), in which the goal is to find a ``Hamiltonian cycle visiting a collection of vertices with the property that exactly one vertex from each [color] is visited''~\cite{pop2024comprehensive}.
If each vertex can belong to several color classes, the instance can be transformed into an instance of \prob{GTSP}~\cite{lien1993tranformation,dimitrijevic1997anefficient}.
However, in the \gi{} problem, we do not demand that a vertex cannot be
visited several times; indeed, we expect that to be the case for many
real-world cases.
Rice and Tsotras~\cite{rice2012exact} gave an exact algorithm
for \prob{GTSP} in $\Oh^*(2^k)$ time\footnote{The $\Oh^*$ notation hides polynomial factors.}, and although being inapproximable to a logarithmic factor~\cite{safra2006complexity}, they gave an $O(r)$-approximation in running time $\Oh^*(2^{k/r})$~\cite{rice2013parameterizedalgorithms}.

Two other related problems are the \textsc{$T$-Cycle} problem, and the \textsc{Maximum Colored $(s, t)$-Path} problem\footnote{
Also studied under the names \textsc{Tropical Path}~\cite{cohen2021tropical} and \textsc{Maximum Labeled Path}~\cite{couetoux2017maximum}.}.  Björklund, Husfeldt, and Taslaman provided an $\Oh^*(2^{|T|})$ randomized algorithm for the \textsc{$T$-Cycle} problem, in which we are asked to find a (simple) cycle that visits all vertices in $T$~\cite{bjorklund2012shortest}.
In the \textsc{Maximum Colored $(s, t)$-Path} problem, we are given a
vertex-colored graph $G$, two
vertices $s$ and $t$, and an integer $k$, and we are asked if there
exists an $(s,t)$-path that collects at least $k$ colors, and if so,
return one with minimum weight.
Fomin et al.~\cite{fomin2023fixed} gave a randomized algorithm running
in time $\Oh^*(2^k)$ for this problem.
Again, in both of these problems a crucial restriction is the search for \emph{simple} paths or cycles.

Though \gi{} is distinct from the problems mentioned above, we leverage techniques from this literature to propose two algorithms which can solve \gi{} optimally\footnote{Note that in this case, and subsequently in the paper unless otherwise indicated, `optimal' refers to an optimal walk on the given graph and is distinct from the asymptotic optimality guarantees provided in~\cite{fu2021computationally}.}.
First, in~\Cref{sec:fpt} we show that while \gi{} is \cclass{NP}-hard (as a \TSPshort{} generalization), a dynamic programming approach can solve our problem in $2^{\colorsetsize} \cdot \text{poly}(n)$
time and memory, where $\colorsetsize$ is the number of POIs.
Additionally, we draw on techniques used in the study of~\TSP{}~\cite{cohen2019severalgraph}
to provide a novel integer linear programming (ILP) formulation, which we describe in~\Cref{sec:ilp-theory}.

To deal with the computational intractability of \gi{}, Fu \emph{et al.} took the approach of relaxing the problem to \emph{near optimality},
enabling them to leverage heuristics in the graph search to achieve reasonable computational speed when solving the problem.
We take two approaches. Using the ILP, we show that on several practical instances drawn from~\cite{fu2021computationally},
\gi{} can be solved almost exactly in reasonable runtime.
However, we note the ILP will not directly scale to very large graphs. For the dynamic programming routine, our approach is more nuanced:
First, we note that as the running time and memory consumption of this algorithm is exponential only in the number of POIs, while remaining
polynomial in the size of the graph, it can optimally solve \gi{} when only a few POIs are present. 
This situation naturally arises in some application areas, particularly in medicine when the most relevant anatomical POIs may be few and known in advance. When the number of POIs is large, we adapt the dynamic programming algorithm into a heuristic by selecting several small subsets of POIs in a principled manner (see~\Cref{sec:color-reduction}), running the dynamic program independently for each small subset, and then ``merging'' the resulting walks (see~\Cref{sec:walk-merging}). Though our implementation is heuristic, it is rooted in theory: it is possible to combine dynamic programming with a ``partition and merge'' strategy such that, given enough runtime, the resulting walk is optimal (see~\Cref{appendix:B}).

We demonstrate the practical efficacy of our algorithms on \gi{} instances
drawn from two scenarios which were used to evaluate the prior state-of-the-art planner \iriscli{}~\cite{fu2021computationally}.
The first is planning inspection for a bridge using a UAV (the ``drone'' scenario), and the second is
planning inspection of the inside of a patient's pleural cavity using a continuum medical robot (the ``crisp'' scenario).
We implemented our algorithms, \dpipa{} (Dynamic Programming) and \ilpipa{} (ILP),
where IPA stands for Inspection Planning Algorithm.
We show (see~\Cref{sec:compare-to-iris} and~\Cref{fig:headline}) that on \gi{} instances of sizes similar to those used by~\cite{fu2021computationally},
\ilpipa{} produces walks with lower weight and higher coverage than those produced by \iriscli{}. Indeed, \ilpipa{} can produce walks with perfect coverage, even on
much larger instances. However, for these larger instances \dpipa{} provides a compelling alternative, producing walks with much lower weight, with some sacrifice in coverage. 

In summary, this work takes steps toward the application of \gi{} as a problem formalization for inspection planning,
and importantly provides (i) multiple novel algorithms with quality guarantees, (ii) an extensive discussion of
methods used to implement these ideas in practice, and (iii) reusable software which outperforms the state-of-the-art on
two relevant scenarios from the literature.\looseness-1

\section{Preliminaries}\label{sec:prelims}

All graphs $G = (V, E)$ in this work are undirected, unless explicitly stated otherwise.
We denote the vertices and edges of~$G$ by~$V$ and~$E$, respectively.
When it is clear which graph is referenced from context,
we write $n=\abs{V}$ for the number of vertices and $m=\abs{E}$ for the number of edges in the graph.
An edge-weight function $w$ is a function $w \colon E \rightarrow \R_{\geq 0}$.
A (simple) \emph{path}~$P = v_1, v_2, \dots, v_t$ in $G = (V,E)$ is a sequence of vertices such that for every $i < t$, it holds that~$v_i v_{i+1} \in E$ and no vertex appears more than once in $P$.
If we relax the latter requirement, we call the sequence~$P$ a \emph{walk}.
A walk is \emph{closed} if it starts and ends in the same vertex.

The \emph{weight} of a walk is the sum of weights of its edges.
For vertices $u,v \in V$, we let $d(u,v)$ denote the \emph{distance}, that is, the minimum weight of any path between~$u$ and~$v$.
For a set $S \subseteq V$, let $d(v,S)=d(S,v)$ be~$\min_{u \in S}d(u,v)$.
Moreover, we denote the graph induced by~$S$ by~$G[S]$ and we use~$G-S$ as a shorthand for~$G[V \setminus S]$.
When~$S$ only contains a single vertex~$v$, then we also write~$G-v$ instead of~$G-\{v\}$.
We refer to the textbook by Diestel~\cite{diestel2005graph} for an introduction to graph theory.

We use $\col(v)$ to denote the \emph{colors} (or \emph{labels}) of a vertex (which, with nuance described below, correspond to POIs in the robot's workspace), and we write $\col(S)$ to denote $\bigcup_{v \in S}\col(v)$.
For a set $S$, we write $2^S$ for the power set of $S$.
We next define the main problem we investigate in this paper.

\begin{problembox}{Graph Inspection}
    \Input & An undirected graph~$G=(V,E)$, a set~$\colorset$ of colors,
    an edge-weight function~$w \colon E \rightarrow \R_{\geq 0}$,
    a coloring function~$\col \colon V \rightarrow 2^\colorset$,
    a start vertex~$s \in V$, and an integer~$t$.\\
    \Prob  & Find a closed walk $P= (v_0,v_1,\ldots,v_{p})$ in $G$
    with $v_0=v_p=s$ and~$|\bigcup_{i=1}^p \col(v_i)|\geq t$
    minimizing~$\sum_{i=1}^{p} w(v_{i-1} v_i)$.\\
  \end{problembox}

Note that $t$ is the minimum number of colors to collect.
For the sake of simplicity, we may assume that~$G$ is connected, $t \leq |C|$, and~$\col(s) = \emptyset$.

\paragraph{Parameterized complexity.} A \emph{parameterized} problem is a tuple~$(I, k)$ where $I \in \Sigma^*$ is the input instance for a set of strings $\Sigma^*$ and $k \in \mathbb N$ is a \emph{parameter}.  A parameterized problem is \emph{fixed-parameter tractable} (FPT) if there exists an algorithm solving every instance~$(I, k)$ in $f(k) \cdot \text{poly}(|I|)$ time, where $f$ is any computable function.  We refer to the textbook by Cygan et al.~\cite{cygan2015parameterized} for an introduction to parameterized complexity theory.
The goal of parameterized algorithms is to capture the exponential (or even more costly) part of the problem complexity within a parameter, making the rest of the computation polynomial in the input size.

\section{Graph Search}\label{sec:methods}

In this section, we present two algorithms for solving \gi, along with
strategies for finding upper and lower bounds on the optimal solution.

\subsection{Dynamic Programming Algorithm}\label{sec:fpt}

We begin by establishing that \gi~is fixed-parameter tractable with respect to the number of colors by giving a dynamic programming algorithm we refer to as \dpipa{}. 

\begin{theorem}
    \label{thm:DP}
    \gi{} can be solved in~$\Oh((2^{\colorsetsize}(n+ \colorsetsize)+ m + n \log n )n)$ time,
    where~$n=|V|$, $m=|E|$, and $\colorset$ is the set of colors.
\end{theorem}

\begin{proof}
  Recall that we may assume $\chi(s)=\emptyset$.
  Otherwise, we can collect all colors in~$\chi(s)$ for free;
  removing the colors~$\chi(s)$ from the coloring function and decreasing $t$ by~$|\chi(s)|$
  gives an equivalent instance.
  We solve \gi{} using dynamic programming.
  First, we compute the all-pairs shortest paths of $G$ in~$\Oh(nm+n^2\log n)$ by
  $n$ calls of Dijkstra's algorithm ($\Oh(m+n\log n)$ time) using a Fibonacci heap.
  Note that the new distance function~$w'$ is complete and metric.
  Hence, we may assume that an optimal solution collects at least one new color in each step
  (excluding the last step where it returns to~$s$).
  Hence, we store in a table~$T[v,S]$ with~$v \in V \setminus \{s\}$ and~$S \subseteq \colorset$, where~$S$ contains at least one color in~$\col(v)$, the length of a shortest walk that starts in~$s$, ends in~$v$, and collects (at least) all colors in~$S$.
  We fill~$T$ by increasing size of~$S$ by the recursive relation:
  \[T[v,S] = \begin{cases}
      \infty & \text{ if } S \cap \col(v) = \emptyset \\
      \min_{u\in V}T[u,S \setminus \col(v)] + w'(uv) & \text{ otherwise.}
    \end{cases}
  \]
  Therein, we assume that~$T[s,S] = 0$ if~$S = \emptyset$ and~$T[s,S]=\infty$, otherwise.
  We will next prove that the table is filled correctly.
  We do so via induction on the size of~$S$.
  To this end, assume that~$T$ was computed correctly for all entries where the respective set~$S$ has size at most~$i$.
  Now consider some entry~$T[v,S]$ where~${S \cap \col(v) \neq \emptyset}$ and~$|S| = i+1$.
  Let~$\ell'$ be the value computed by our dynamic program and let~$\opt$ be the length of a shortest walk that starts in~$s$, ends in~$v$, and collects all colors in~$S$.
  It remains to show that~$\ell' = \opt$ and to analyze the running time.
  We first show that~$\ell' \leq \opt$.
  To this end, let~${W = (s,v_1,\ldots,v_p=v)}$ be a walk of length~$\opt$ that collects all colors in~$S$.
  If~$p = 1$, then~$S \subseteq \col(v)$ and~$\ell' \leq T[s,\emptyset] + w'(sv) = w'(sv)$.
  Moreover, the shortest path from~$s$ to~$v$ has length~$w'(sv)$ and hence~$\ell' \leq w'(sv) \leq \opt$.
  If~$p > 1$, then~$W' = (s,v_1,\ldots,v_{p-1})$ is a walk from~$s$ to~$v_{p-1}$ that collects all colors in~$S' = S \setminus \col(v)$.
  Hence, by construction~$T[v_{p-1},S'] \leq \opt - w'(v_{p-1} v)$ and hence~$\ell' = T[v,S] \leq T[v_{p-1},S'] + w'(v_{p-1} v) \leq \opt$.

  We next show that~$\ell' \geq \opt$.
  To this end, note that whenever~$T[v,S]$ is updated, then there is some vertex~$u$ such that~${T[v,S] = T[u,S \setminus \col(v)] + w'(uv)}$ (where possibly~$u = s$ and~$S \setminus \col(v) = \emptyset$).
  By induction hypothesis, there is a walk from~$s$ to~$u$ that collects all colors in~$S \setminus \col(v)$ of length~$T[u,S\setminus \col(v)]$.
  If we add vertex~$v$ to the end of this walk, we get a walk of length~$\ell'$ that starts in~$s$, ends in~$v$, and collects all colors in~$S$.
  Thus, $\opt \leq \ell'$.

  After the table is completely filled for all color sets~$S$ with~$|S| \leq t$, then we just need to check whether there exists a vertex~$v \in V\setminus \{s\}$ and a set~$S$ with~$|S| = t$ such that~$T[v,S] + w'(vs) \leq \ell$.
  Note that the number of table entries is~$2^{\colorsetsize}n$, computing one table entry takes~$\Oh(n+ \colorsetsize)$~time, and the final check in the end takes~$\Oh(2^{\colorsetsize}n)$ time.
  Thus, the overall running time of our algorithm is in~$\Oh(nm + n^2\log(n)+2^{\colorsetsize}(n+\colorsetsize)n) = \Oh((2^{\colorsetsize}(n+\colorsetsize)+m+n\log n)n)$.
  \qed
\end{proof}

\subsection{Integer Linear Programming Algorithm}\label{sec:ilp-theory}

In this section, we present \ilpipa, an Integer Linear Programming (ILP) formulation of the problem,
inspired by the flow-based technique for \TSPshort~\cite{cohen2019severalgraph}.
As a (trivially checkable) precondition, we require that a solution walk includes at least two vertices.
We observe that there is a simpler formulation if the input is a complete metric graph; however, in practice, creating the completion and applying this approach significantly degrades performance because of the runtime's dependence on the number of edges. 

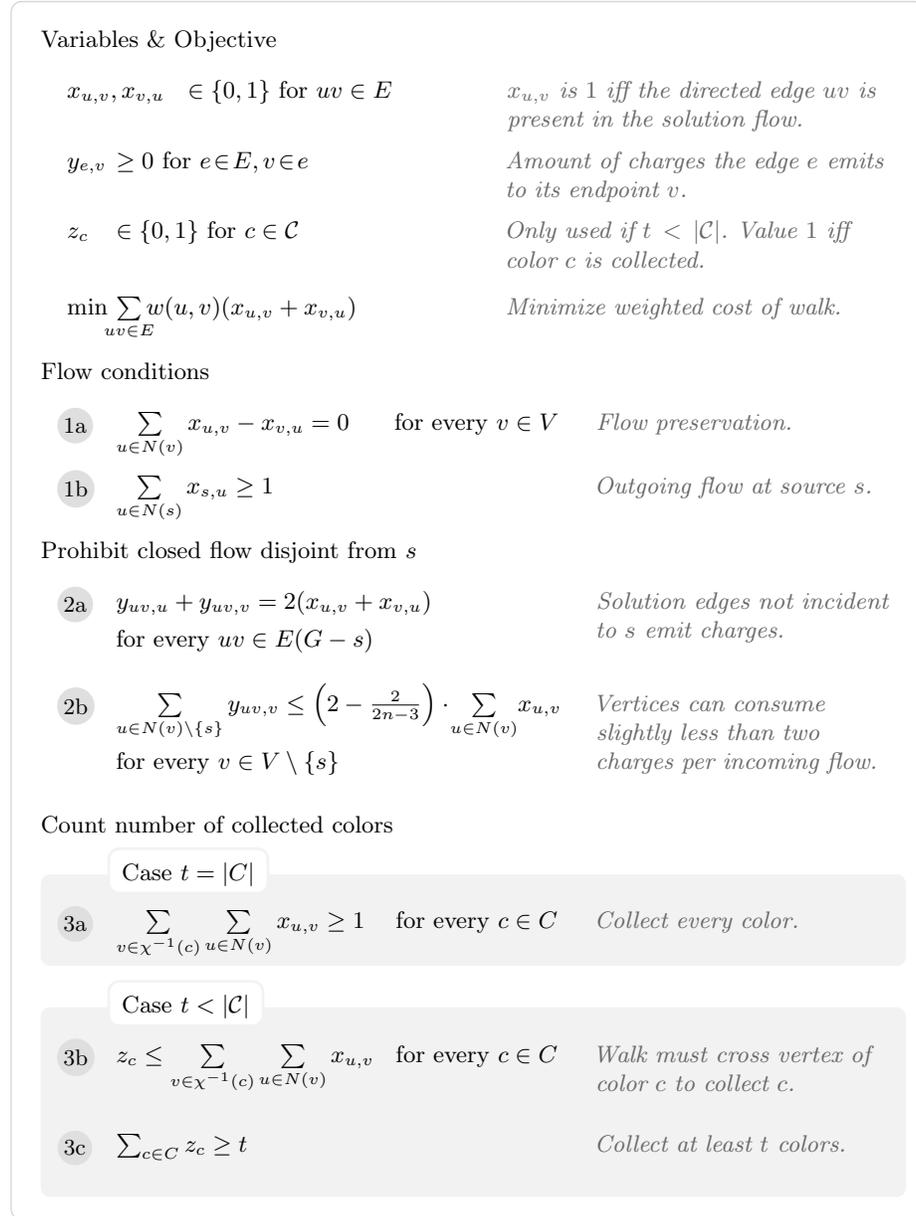
\begin{figure}[h!]
\vspace{-1em}
    \begin{GrayBox}{}
  \begin{tikzpicture}[node distance=1cm]
      \tikzset{
          every node/.style={{anchor=south west}, text depth=0pt},
          const/.style={},
          constLabel/.style={circle,fill=gray!25,inner sep=1.2pt, scale=0.95, font={\small}},
          comment/.style={font={\footnotesize\itshape},text width=3.9cm,color=black!60},
          longcomment/.style={font={\footnotesize\itshape},text width=5cm,color=black!60},
          headline/.style={font={}},
          mult/.style={font={\small}}
      }
      \tikzmath{
          \gridlineZero = 0; 
          \gridlineNumbers = 0.4;
          \gridlineVars = \gridlineNumbers-0.05;
          \gridlineConstraint = 1;
          \gridlineObj = 2;
          \gridlineMult = 4.7;
          \gridlineLongComment = 6.2;
          \gridlineComment = 7.4;
          \y = 0; 
          \lh = 0.85;
          \debugheight = 14;
      } 


      \draw (\gridlineZero, \y) node[headline] (headVars) {Variables \& Objective};

      \tikzmath{ \y = \y - .8*\lh; }
      \draw (\gridlineVars, \y) node {$x_{u,v}, x_{v,u}$};
      \draw (\gridlineConstraint, \y) node {\hspace*{7ex}$\in \{0,1\}$ \footnotesize for $uv \in E$};
      \draw (\gridlineLongComment, \y) node[longcomment] {$x_{u,v}$ is $1$ iff the directed edge $uv$ is present in the solution flow.}; 

      \tikzmath{ \y = \y - 1.1*\lh; }
      \draw (\gridlineVars, \y) node {$y_{e,v}$};
      \draw (\gridlineConstraint, \y) node {$\geq 0$ \footnotesize for $e \!\in\! E, v \!\in\! e$};
      \draw (\gridlineLongComment, \y) node[longcomment] {Amount of charges the edge~$e$ emits to its endpoint~$v$.};

      \tikzmath{ \y = \y - 1.1*\lh; }
      \draw (\gridlineVars, \y) node {$z_c$};
      \draw (\gridlineConstraint, \y) node {$\in \{0,1\}$ \footnotesize for $c \in \colorset$};
      \draw (\gridlineLongComment, \y) node[longcomment] {Only used if~$t < \colorsetsize$. Value~$1$ iff color~$c$ is collected.};

      \tikzmath{ \y = \y - 1.2*\lh; }

      \draw (\gridlineVars, \y) node (obj) {$
             \min \!\! \sum\limits_{uv \in E} \!\!\! w(u, v) (x_{u,v} + x_{v,u})
      $};    
      \draw (\gridlineLongComment,\y) node[longcomment] {Minimize weighted cost of walk.};    

      \tikzmath{ \y = \y - 1*\lh; } 
      \draw (\gridlineZero, \y) node[headline] (headflow) {Flow conditions};

      \tikzmath{ \y = \y - 0.8*\lh; } 
      \draw (\gridlineNumbers,\y) node[constLabel] {1a};
      \draw (\gridlineConstraint,\y) node[const] {$
          \sum\limits_{u \in N(v)} x_{u,v} - x_{v,u} = 0
      $};
      \draw (\gridlineMult,\y) node[mult] {
          for every $v \in V$
      };
      \draw (\gridlineComment,\y) node[comment] {Flow preservation.};    

      \tikzmath{ \y = \y - \lh; } 
      \draw (\gridlineNumbers,\y) node[constLabel] {1b};
      \draw (\gridlineConstraint,\y) node {$
          \sum\limits_{u \in N(s)} x_{s,u} \geq 1
      $};    
      \draw (\gridlineComment,\y) node[comment] {Outgoing flow at source~$s$.};

      \tikzmath{ \y = \y - 1*\lh; } 
      \draw (\gridlineZero, \y) node[headline] (headflow) {Prohibit closed flow disjoint from~$s$};

      \tikzmath{ \y = \y - 0.8*\lh; } 
      \draw (\gridlineNumbers,\y) node[constLabel] {2a};
      \draw (\gridlineConstraint,\y) node {$\displaystyle
          y_{uv,u} + y_{uv,v} = 2(x_{u,v} + x_{v,u})
      $};    
      \draw (\gridlineComment,\y) node[comment] {Solution edges not incident to $s$ emit charges.};
      \tikzmath{ \y = \y - .6*\lh; } 
      \draw (\gridlineConstraint,\y) node[mult] {
          for every $uv \in E(G - s)$
      };    
   
      \tikzmath{ \y = \y - 1*\lh; } 
      \draw (\gridlineNumbers,\y) node[constLabel] {2b};
      \draw (\gridlineConstraint,\y) node {$
          \sum\limits_{u\in N(v) \setminus \{s\}} y_{uv,v} \leq \left(2 - \frac{2}{2n-3}\right)\cdot  \! \sum\limits_{u \in N(v)} \! x_{u,v}
      $};       
      \draw (\gridlineComment,\y) node[comment] {
          Vertices can consume slightly less than two charges per incoming flow.
      }; 
      \tikzmath{ \y = \y - 0.9*\lh; } 
      \draw (\gridlineConstraint,\y) node[mult] {
          for every $v \in V \setminus \{s\}$
      };

      \tikzmath{ \y = \y - \lh; } 
      \draw (\gridlineZero, \y) node[headline] (headcol) {Count number of collected colors};

      \tikzmath{ \y = \y - .5*\lh; } 
      
      \tikzmath{ \y = \y - \lh; } 
      \draw (\gridlineNumbers,\y) node[constLabel] (rectAStart) {3a};
      \draw (\gridlineConstraint,\y) node {$
          \sum\limits_{v \in \col^{-1}(c)} \sum\limits_{u \in N(v)} x_{u,v} \geq 1
      $};    
      \draw (\gridlineMult2,\y) node[mult] {
          for every $c \in C$
      };
      \draw (\gridlineComment,\y) node[comment] (rectAEnd) {
          Collect every color.
      }; 

      \begin{pgfonlayer}{bg} 
          \draw[white,thick,fill=black!5,rounded corners] ($(rectAStart.north west)+(-0.3,0.5)$)  rectangle ($(rectAEnd.south east)+(0.1,-0.4)$);

          \coordinate (temp) at (rectAStart.north -| \gridlineConstraint,0);
          \draw ($(temp)+(0,0.15)$) node[rectangle,fill=white,rounded corners,inner sep=5pt,draw=black!5,line width=1pt] {\small Case $t = |C|$};        
      \end{pgfonlayer}
      
      \tikzmath{ \y = \y - 1.1*\lh; } 

      \tikzmath{ \y = \y - \lh; } 
      \draw (\gridlineNumbers,\y) node[constLabel] (rectBStart) {3b};
      \draw (\gridlineConstraint,\y) node {$
          z_c \leq \sum\limits_{v \in \col^{-1}(c)} \sum\limits_{u \in N(v)} x_{u,v} 
      $};    
      \draw (\gridlineMult2,\y) node[mult] {
          for every $c \in C$
      };    
      \draw (\gridlineComment,\y) node[comment] {
          Walk must cross vertex of color~$c$ to collect~$c$.
      }; 

      \tikzmath{ \y = \y - .4*\lh; } 

      \tikzmath{ \y = \y - \lh; } 
      \draw (\gridlineNumbers,\y) node[constLabel] {3c};
      \draw (\gridlineConstraint,\y) node {$
          \sum_{c \in C} z_c \geq t
      $};     
      \draw (\gridlineComment,\y) node[comment] (rectBEnd) {
          Collect at least~$t$ colors.
      }; 

      \begin{pgfonlayer}{bg} 
          \draw[white,thick,fill=black!5,rounded corners] ($(rectBStart.north west)+(-0.3,0.5)$)  rectangle ($(rectBEnd.south east)+(0.1,-0.5)$);

          \coordinate (temp) at (rectBStart.north -| \gridlineConstraint,0);
          \draw ($(temp)+(0,0.15)$) node[rectangle,fill=white,rounded corners,inner sep=5pt,draw=black!5,line width=1pt] {\small Case $t < \colorsetsize$};
      \end{pgfonlayer}
  \end{tikzpicture}
  \end{GrayBox}
  \vspace{-1.5em}
  \caption{\ilpipa, an ILP for the \textsc{Graph Inspection} problem.}\label{fig:ilp-formulation}
\end{figure}

Our ILP formulation for \gi{} can be found in~\Cref{fig:ilp-formulation}.
Intuitively, the \emph{flow amount} at a vertex encodes the number of occurrences of the vertex in a walk.
Constraints (1a) and (1b) implement flow conditions, and
(2a) and (2b) ensure that the flow originates at vertex $s$.
An edge included in a solution and not touching $s$ emits $2$ \emph{charges}, and
the charges are distributed among the edge's endpoints.
If every solution edge is part of a walk from $s$,
then a charge consumption at each vertex can be slightly less than $2$ per incoming flow.
There are $\Oh(\colorsetsize + m)$ constraints if $t=\colorsetsize$, and $\Oh(\colorsetsize m)$ constraints if $t < \colorsetsize$.

\subsubsection*{Correctness.}
Before showing the correctness of the ILP formulation,
we characterize solution walks for \gi.
In the following, we view a solution walk as a sequence of directed edges.
For a walk $P=v_0 v_1 \ldots v_\ell$, we write $|P|$ for the number of edges in $P$,
i.e. $|P|=\ell$, and $E(P)$ for the set of \emph{directed} edges in the walk,
i.e. $E(P)=\{v_{i-1}v_i \mid 1 \leq i \leq \ell \}$.
We write $w(P)$ for the length of the walk, that is, $w(P) := \sum_{uv \in E(P)} w(u,v)$. We now prove a simple lemma.

\begin{lemma}\label{lem:distinct-edges}
  For any feasible instance of \textnormal{\gi},
  there exists an optimal solution without repeated directed edges.
\end{lemma}

\begin{proof}
  Assume not, and let $P = s P_1 uv P_2 uv P_3 s$ be a solution minimizing $|P|$.
  Consider a walk $P' = s P_1 u \overline{P_2} v P_3 s$,
  where $\overline{P_2}$ is the reversed walk of $P_2$.
  Since~${|P'|<|P|}$ and both visit the same set of vertices,
  we have that~$w(P')>w(P)$ by our choice of~$P$.
  However,
  $w(P')=w(sP_1u) + w(u\overline{P_2}v) + w(vP_3s) \leq w(sP_1u) + w(u,v) + w(vP_2u) + w(u,v) + w(vP_3s) = w(P)$,
  a contradiction.
  \qed
\end{proof}

The following is a simple observation.

\begin{observation}\label{lem:eulerian}
    Given an instance of \textnormal{\gi{}},
    there exists a closed walk $P$ of length $\hat{w}$ visiting vertices $V' \subseteq V(G)$
    if and only if
    there exists a connected Eulerian multigraph $G'=(V', E')$
    such that $\hat{w} = \sum_{uv \in E'}w(uv)$,
    where $E'$ is the multiset of the edges in $P$.
\end{observation}

This leads to a structural lemma about solutions.

\begin{lemma}\label{lem:num-edges-bound}
  For any feasible instance of \textnormal{\gi},
  there exists an optimal solution with at most $2n-2$ edges.
  This bound is tight.
\end{lemma}

\begin{proof}
  Let $P$ be an optimal solution with the minimum number of edges.
  From Observation \ref{lem:eulerian}, we may assume there exists a connected Eulerian multigraph $H$
  that encodes $P$.
  Since $H$ is connected, it has a spanning tree $T$ as a subgraph.
  Let $H'=H - E(T)$.
  If $H'$ contains a cycle $C$, then $H-C$ is also connected and Eulerian,
  as removing a cycle from a multigraph does not change the parity of the degree of each vertex.
  Hence, there exists a shorter solution $P'$ that is an Eulerian tour in $H-C$,
  a contradiction.
  Knowing that both $T$ and $H'$ are acyclic, we have $|E(H)|=|E(T)|+|E(H')|\leq 2n-2$.
  This is tight whenever $G$ is a tree where all leaves have a unique color.
  \qed
\end{proof}

Now we are ready to prove the correctness of the ILP formulation.

\begin{theorem}
  The ILP formulation in \Cref{fig:ilp-formulation} is correct.
\end{theorem}

\begin{proof}
We show that we can translate a solution for \gi{} to
a solution for the corresponding ILP and vise versa.

For the forward direction, let $P=v_0 v_1 \ldots v_{\ell}$ with~$v_0=v_\ell=s$ be a solution walk with $\ell \geq 2$ collecting at least $k$ colors.
From \Cref{lem:distinct-edges},
we may assume that there are no indices $i,j$ such that $i<j$ and $v_{i}v_{i+1} = v_{j}v_{j+1}$.
For constraint~(1), we set $x_{u,v}=1$ if $uv \in E(P)$ and $0$ otherwise.
It is clear to see that all flow conditions are satisfied.
Moreover, observe that for any vertex $v \in V(G) \setminus \{s\}$,
the flow amount $\sum_{u \in N(v)}x_{u,v}$ corresponds to the number of occurrences of $v$ in~$P$,
which we denote by~$\deg_P(v)$.

Next, if $|P|=2$, then constraint (2) is trivially satisfied by setting $y_{e,v}=0$ for all $e,v$.
Otherwise, let $P'$ be a continuous part of~$P$ such that~$s$ appears only at the beginning and at the end.
Then, $P'$ contains $|P'|-2$ edges that do not touch~$s$ and emit two charges each.
We know that $|P'|-1 = \sum_{v \in V(P') \setminus \{s\}} \deg_{P'}(v)$.
For a directed edge $e \in E(P')$ and its endpoint $v \in e$,
let $y_{e,v}^{(P')}$ be part of $y_{e,v}$ charged only by $P'$.
We distribute the charges by setting $y_{v'_{i-1}v'_{i}, v'_{i-1}}^{(P')} = 2-\frac{2i}{|P'|-1}$ and
$y_{v'_{i-1}v'_{i}, v'_{i}}^{(P')}=\frac{2i}{|P'|-1}$ for every $1 \leq i < |P'|$,
where $P'=v_0' v_1' \ldots v_{|P'|}'$ with~$v'_0 = v'_{|P'|} = s$.

Note that~$\sum_{u \in N(v) \setminus \{s\}}y_{uv, v}^{(P')} =
\deg_{P'}(v) \cdot \frac{2(|P'|-2)}{|P'|-1}
= (2-\frac{2}{|P'|-1}) \cdot \deg_{P'}(v)
\leq (2-\frac{2}{2n-3})\cdot  \deg_{P'}(v)$
for every $v \in V(P') \setminus \{s\}$.
The last inequality is due to \Cref{lem:num-edges-bound}.
This inequality still holds when we concatenate closed walks $P'$ from $s$
since $\sum_{u \in N(v) \setminus \{s\}} y_{uv,v}
= \sum_{P'} \sum_{u \in N(v) \setminus \{s\}} y_{uv,v}^{(P')}$
and $\sum_{u \in N(v)}x_{u,v} = \sum_{P'} \deg_{P'}(v)$.
Constraint (2) is now satisfied.

Finally, in order to collect colors~$\col(v)$, there most be an edge $uv$ in the solution.
Notice that constraint (3b) encodes this and constraint (3c) ensures that we collect at least~$k$ distinct colors.
Finally, observe that the objective is properly encoded.

For the backward direction, we show that there cannot be a closed flow, i.e. \emph{circulation}, avoiding~$s$.
For the sake of contradiction, let $C$ be such a circulation.
Then, since~$x_{u,v}=1$ for every~$uv \in E(C)$,
we have $\sum_{uv=e \in E(C)} y_{e,u} + y_{e,v}=2|E(C)|$.
This is considered as the total charge emitted from $C$,
and it must be consumed by the vertices in $C$.
We have
$\sum_{v \in V(C)} \sum_{u \in N(v)} y_{uv, u} + y_{uv, v} \geq 2|E(C)|$,
and by the pigeonhole principle,
there must be a vertex $v \in V(C)$ such that its charge consumption is at least
$\deg(v)$, violating constraint~(2b).
Hence, there must be a closed walk from $s$ that realizes a circulation obtained by ILP.
From constraint~(3), the walk also collects at least $k$ colors.
\qed
\end{proof}

\subsubsection*{Solution recovery.}
A closed walk in a multigraph is called an \emph{Euler tour} if it traverses every edge of the graph
exactly once.
A multigraph is called \emph{Eulerian} if it admits an Euler tour.
It is known that a connected multigraph is Eulerian if and only if every vertex has even degree \cite{euler1741solutio} and given an Eulerian multigraph with $m$ edges, we can find an Euler tour in time $\Oh(m)$ \cite{hierholzer1873ueber}.

Given a certificate of an optimal solution for the aforementioned ILP, we construct a solution walk as follows.
First, let $D$ be the set of directed edges $uv$ such that $x_{u,v}=1$.
Next, we find an Euler tour $P$ starting from~$s$ using all the edges in $D$.
Then, $P$ is a solution for \prob{Graph Inspection}.

\subsection{Upper and Lower Bounds}\label{sec:theory-bounds}
When evaluating solutions, having upper and lower bounds on the optimal solution provides useful context. 
For \gi, a polynomial-time computable lower bound follows directly from the LP relaxation of the ILP in \Cref{sec:ilp-theory}.
For an upper bound, we consider \textsf{Algorithm ST} (Algorithm \ref{alg:alg-st}),
which uses a 2-approximation algorithm for \prob{Steiner Tree}\footnote{%
The \prob{Steiner Tree} problem takes a graph $G$ and a set of vertices $S$ (called \emph{terminals})
and asks for a minimum-weight tree in $G$ that spans $S$.}~\cite{kou1981fast} as a subroutine.
The algorithm proceeds by first choosing the vertices closest to~$s$ collecting $t$ colors
and then finding a Steiner tree of those vertices.
A closed walk can be obtained by using each edge of the Steiner tree twice.

\begin{algorithm}[H]
  $S \gets \{s\}$\\
  \While{$|\col(S)| < t$}{
      \tcp{Choose the vertex with a new color closest to $s$.}
      $S \gets S \cup \{\arg\min_{v \in V} d(s,v) \mid \col(v) \setminus \col(S) \neq \emptyset \}$
  }
  Compute a $2$-approximation $T$ for \prob{Steiner Tree} on $G$ with terminals $S$\\
  Construct a closed walk from $s$ using the all edges in $T$
  \caption{\textsf{Algorithm ST}}\label{alg:alg-st}
\end{algorithm}

\begin{restatable}{theorem}{restatealgorithmst}\label{thm:algorithm-st}
    \textnormal{\textsf{Algorithm ST}} returns a closed walk collecting at least $t$
    colors with length at most $t \cdot \textnormal{opt}$,
    where $\textnormal{opt}$ denotes the optimal walk length.
    The algorithm runs in time $\Oh(tm \log (n+t))$.
\end{restatable}
\begin{proof}
    Since the algorithm returns a walk including all vertices in $S$,
    it collects at least $t$ colors.
    Let~$d_c = \min_{v\in \chi^{-1}(c)} d(s,v)$ for every $c \in \colorset$.
    Then, let $\tilde{d}$ be the $t$-th smallest such value,
    and due to Steps 1-3, for every $u \in S$, we have $d(s,u) \leq \tilde{d}$.
    Since $|S| \leq t+1$, the weight of the minimum Steiner tree is at most $t \tilde{d}$,
    which results in that the length~$\ell'$ of the walk returned by our algorithm is at most~$2t \tilde{d}$.
    Now, suppose that $P$ is an optimal walk of length $\text{opt}$ collecting at least $t$ colors $\colorset'$.
    Then, it is clear to see that $\text{opt} \geq 2 \cdot d_c$ for any~$c \in \colorset'$.
    From $|\colorset'| \geq t$, we have $\text{opt} \geq 2 \tilde{d}$, which implies~$\ell' \leq t \cdot \text{opt}$.
  
  We next analyze the running time.
    Steps 1-3 takes $\Oh(m \log n + tn \log t)$ time for 
    sorting vertices and computing the union of colors.
    Step 4 can be done by computing the transitive closure on $S$,
    which takes $\Oh(tm \log n)$ time.
    Step 5 takes $\Oh(n+m)$ time, so the overall running time is in~$\Oh(tm \log (n+t))$.
  
    Lastly, we show that this bound cannot be smaller.
    Let $G$ be a star $K_{1,t+1}$ with $s$ being the center with no colors.
    One leaf~$u$ has $t$ colors, and each of the other $t$ leaves has a unique single color.
    Every edge has weight $1$.
    The optimal walk is $(s,u,s)$ and has length $2$,
    whereas the algorithm may choose $V \setminus \{u\}$ as $S$.
    This gives a walk of length $2t$.
    \qed
  \end{proof}

\section{Graph Simplification}\label{sec:simplification}
  
\begin{figure}[t]
  \centering

  \tikzstyle{small} = [circle, fill=white, text=black, draw, thick, scale=1, minimum size=0.3cm, inner sep=1.5pt, line width=0.5mm]

  \definecolor{kOrange} {RGB} {230, 159, 0}
  \definecolor{kSky} {RGB} {86, 180, 233}
  \definecolor{kGreen} {RGB} {0, 158, 115}
  \definecolor{kYellow} {RGB} {240, 228, 66}
  \definecolor{kBlue} {RGB} {0, 114, 178}
  \definecolor{kRed} {RGB} {220, 51, 17}
  \definecolor{kPurple} {RGB} {204, 121, 167}
  \definecolor{kGray} {RGB} {187, 187, 187}
  \definecolor{kLightGray} {RGB} {221, 221, 221}
  \definecolor{kDarkGray} {RGB} {85, 85, 85}

  \begin{minipage}[m]{.32\textwidth}
      \vspace{0pt}
      \centering
      \begin{tikzpicture}
          \node[small,fill=kBlue] (v11) at (1, 1) {};
          \node[small,fill=kOrange] (v12) at (1, 2) {};
          \node[small,fill=kGreen] (v13) at (1, 3) {};
          \node[small] (v22) at (2, 2) {};
          \node[small] (v23) at (2, 3) {$s$};
          \node[small,fill=kYellow] (v31) at (3, 1) {};
          \node[small,fill=kDarkGray] (v32) at (3, 2) {};
          \node[small,fill=kPurple] (v33) at (3, 3) {};

          \draw (v11) -- (v12);
          \draw (v12) -- (v13);
          \draw (v11) -- (v22);
          \draw (v22) -- (v23);
          \draw (v31) -- (v22);
          \draw (v31) -- (v32);
          \draw (v32) -- (v33);
          \draw (v13) -- (v23);
          \draw (v23) -- (v33);

          \node[] at (2, 3.8) {Original graph};
      \end{tikzpicture}
  \end{minipage}
  \hfill
  \begin{minipage}[m]{.32\textwidth}
    \vspace{0pt}
    \centering
    \begin{tikzpicture}
      \draw[line width=1.2mm,draw=kBlue] (2,3)--(1,3)--(1,1)--(1.9,2)--(1.9,3);
      \draw[line width=1.2mm,draw=kRed] (2,3)--(3,3)--(3,1)--(2.1,2)--(2.1,3);

      \node[small,fill=kBlue] (v11) at (1, 1) {};
      \node[small] (v12) at (1, 2) {};
      \node[small,fill=kGreen] (v13) at (1, 3) {};
      \node[small] (v22) at (2, 2) {};
      \node[small] (v23) at (2, 3) {$s$};
      \node[small,fill=kYellow] (v31) at (3, 1) {};
      \node[small] (v32) at (3, 2) {};
      \node[small,fill=kPurple] (v33) at (3, 3) {};

      \draw (v11) -- (v12);
      \draw (v12) -- (v13);
      \draw (v11) -- (v22);
      \draw (v22) -- (v23);
      \draw (v31) -- (v22);
      \draw (v31) -- (v32);
      \draw (v32) -- (v33);
      \draw (v13) -- (v23);
      \draw (v23) -- (v33);

      \draw[rounded corners, dashed] (0.7, 0.7) rectangle ++ (0.6, 2.6);
      \node[] at (0.4, 1) {$\mathcal{C}_1$};
      \draw[rounded corners, dashed] (2.7, 0.7) rectangle ++ (0.6, 2.6);
      \node[] at (3.6, 1) {$\mathcal{C}_2$};
      \node[] at (2, 3.8) {Color-reduced graph};
      \end{tikzpicture}
  \end{minipage}
  \hfill
  \begin{minipage}[m]{.32\textwidth}
    \vspace{0pt}
    \centering
    \begin{tikzpicture}
      \draw[line width=1.2mm,draw=kGreen] (2,3)--(1,3)--(1,1)--(2,2)--(3,1)--(3,3)--(2,3);

      \node[small,fill=kBlue] (v11) at (1, 1) {};
      \node[small] (v12) at (1, 2) {};
      \node[small,fill=kGreen] (v13) at (1, 3) {};
      \node[small] (v22) at (2, 2) {};
      \node[small] (v23) at (2, 3) {$s$};
      \node[small,fill=kYellow] (v31) at (3, 1) {};
      \node[small] (v32) at (3, 2) {};
      \node[small,fill=kPurple] (v33) at (3, 3) {};

      \draw (v11) -- (v12);
      \draw (v12) -- (v13);
      \draw (v11) -- (v22);
      \draw (v22) -- (v23);
      \draw (v31) -- (v22);
      \draw (v31) -- (v32);
      \draw (v32) -- (v33);
      \draw (v13) -- (v23);
      \draw (v23) -- (v33);

      \node[] at (2, 3.8) {Merged walk};
      \end{tikzpicture}
  \end{minipage}
  \caption[Example of the partition-and-merge framework]{%
  An illustration of the partition-and-merge framework.
  The color set in the original graph (left) is
  reduced to $2$ color sets $\mathcal{C}_1$ and $\mathcal{C}_2$,
  each of which contains $2$ colors (middle).
  For each color set, we find an optimal walk collecting all colors in the set,
  resulting in the blue and red walks.
  Those walks are merged into the green walk,
  collecting the same colors in the color-reduced graph (right).
  }
  \label{fig:partition-example}
  \vspace*{-1em}
\end{figure}
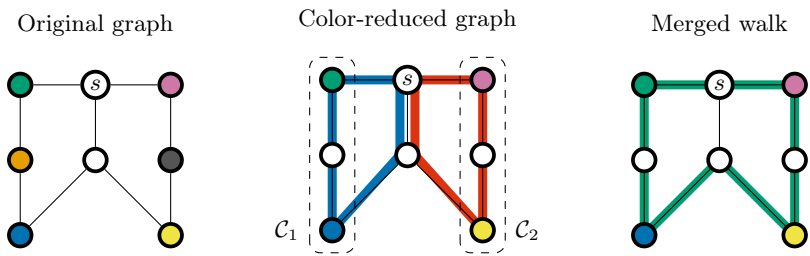

This section introduces strategies for transforming our exact algorithms into heuristics
with improved scalability,
including principled sub-sampling of colors (POIs) and creating plans by merging walks
which inspect different regions of the graph. 
\Cref{fig:partition-example} illustrates this idea, the partition-and-merge framework.

\subsection{Color Reduction}\label{sec:color-reduction}

The algorithms of Sections \ref{sec:fpt} and \ref{sec:ilp-theory} give us the ability to exactly solve \gi{}, but the
running time is exponential in the number of colors (i.e., POIs). In some applications, this may not be
prohibitive. In surgical robotics, a doctor may identify a small number of POIs which
must be inspected to enable surgical intervention. However in other settings, it is unrealistic to
assume that the number of colors is small. In the datasets we explore in~\Cref{sec:experiments} for instance, 
the POIs are drawn from a mesh of the object to be inspected, and we have no a priori information about the
relative importance of inspecting individual POIs.

To deal with this challenge, we find a ``representative'' set~$\colorset' \subseteq \colorset$
of colors, with~$|\colorset'| = k$ small enough that our FPT algorithms run efficiently on the instance in which vertex colors are defined by~$\col(v) \cap \colorset'$
for each vertex $v$. We can then find a minimum-weight walk $P$ on the color-reduced
instance and reconstruct the set of inspected colors by computing~$\bigcup_{v \in P} \col(v)$.

Formally, we assume that there exists some function $\colorsim \colon \colorset^2 \rightarrow \R_{\geq 0}$
which encodes the ``similarity'' of colors, that is, for colors $c_1, c_2, c_3$,
if ${\colorsim(c_1, c_2) < \colorsim(c_1, c_3)}$, then $c_1$ is more similar to $c_2$ than it is to $c_3$.
In this paper, the function~$\colorsim$ is always a Euclidean distance, but we emphasize that our
techniques apply also to other settings. For example, one may imagine applications in which
POIs are partitioned into categorical \emph{types}, and it is desirable that some POIs of each type
are inspected. In this case, one could define $\colorsim$ as an indicator function which returns $0$ or $1$ according
to whether or not the input colors are of the same type. The core idea behind our methods is to select a small set of
colors having \emph{maximum dispersal}, meaning that as much as possible, every color in $\colorset$ should be highly similar
(according to the function $f$) to at least one representative color in $\colorset'$.\looseness-1

\begin{figure}[t]
  \vspace*{-1.1em}
  \input{figures/greedyMDpseudocode.tex}
  \vspace*{-2.4em}
\end{figure}

We evaluate four algorithms for this task. The baseline (which we call \randmd) selects colors uniformly at random.
This is the strategy employed by \iriscli{} when needed~\cite{fu2021computationally}.
The second (called \greedymd---MD for Maximum Dispersal) is a greedy strategy based on the Gonzalez algorithm for $k$-center~\cite{gonzalez1985clustering};
this algorithm is described in more detail in Algorithm~\ref{alg:greedyMD},
where we set $\col_0=\col(s)$.
The final two algorithms (\metricmd{}, \outliermd{}) are modified versions of this strategy.
The interested reader is referred to~\Cref{appendix:color-reduction} for a detailed description
and the results of our comparative study.
We note that all of our algorithms outperform the baseline \randmd{}
in terms of the resulting coverage.
We perform our final comparisons (see~\Cref{sec:compare-to-iris}) using \greedymd{} for color reduction.

\subsection{Merging Walks}\label{sec:walk-merging}

When using \dpipa{} or \ilpipa{} on a color-reduced graph,
the computed walk is minimum weight for the reduced color set,
but the corresponding walk in the original graph may not collect many additional colors.
To increase the coverage in the original graph, we merge two or more walks into a single closed walk.
Now, the challenge is how to keep the combined walk short.
Suppose we have a collection of $W$ solution walks $\{P_i\}$ and want a combined walk $P$ that visits all vertices in $\bigcup_i V(P_i)$.
We model this task as the following problem\footnote{
  An underlying simple graph of a multigraph is obtained by deleting loops and replacing multiedges with single edges.
}.

\begin{problembox}{Minimum Spanning Eulerian Subgraph}
  \Input & A loopless connected Eulerian multigraph $G$
    and edge weights $w: E \to \R_{\geq 0}$, where $E$ denotes the edges in
    $G$'s underlying simple graph.\\
  \Prob & Find a spanning subgraph $G'$ of $G$ such that $G'$ is a connected Eulerian multigraph
  minimizing the weight sum, \textit{i.e.} $\sum_{e \in E(G')}w(e)$.\\
\end{problembox}

We showed in~\Cref{sec:ilp-theory} (see Observation~\ref{lem:eulerian}) that each solution walk for \gi{} is an Eulerian tour in
a multigraph.
We hence use these two characterizations interchangeably.
Unfortunately, this problem is \cclass{NP}-hard
as we can see by reducing from \prob{Hamiltonian Cycle},
which asks to find a cycle visiting all vertices in a graph.
Given an instance $G$ with~$n$~vertices of \prob{Hamiltonian Cycle},
we duplicate all the edges in $G$ so that the graph becomes Eulerian.
If we set a unit weight function $w$ for $E(G)$, i.e. $w(e)=1$ for every~$e \in E(G)$,
then $G$ has a Hamiltonian cycle if and only if $(G,w)$ has a spanning subgraph of weight $n$.

In this paper, we propose and evaluate three simple heuristics for \prob{Minimum Spanning Eulerian Subgraph}: \concatmerge, \greedymerge, and \exactmerge.
\concatmerge simply concatenates all walks.
Since all walks start and end at vertex $s$, their concatenation is also a closed walk.
In \Cref{appendix:B}, we give an algorithm which uses \concatmerge and solves \gi{} optimally. We implemented a simplified version which is better by a factor of~$n$ in both running time and memory usage.
\greedymerge, detailed in \Cref{sec:merge-proof}, is a polynomial-time heuristic including simple preprocessing steps for \prob{Minimum Spanning Eulerian Subgraph}.
At a high level, \greedymerge builds a minimum spanning tree and removes as many redundant cycles as possible from the rest.
\exactmerge is an exact algorithm using the ILP formulation for \gi{}.

\subsubsection*{Algorithm \exactmerge.}
We construct an instance of \gi{} by taking the underlying simple graph of the instance $G$
of \prob{Minimum Spanning Eulerian Subgraph}.
We pick an arbitrary vertex $s \in V(G)$ as the starting vertex, and set unique colors to the other vertices.
After formulating the ILP for \gi{} with~$t=n-1$ (collecting all colors) as in \Cref{sec:ilp-theory},
we add the following constraints:
$x_{u,v} + x_{v,u} \leq 1$ for every edge $uv \in E(G)$ with multiplicity $1$.
Lastly, we map a solution for the ILP to the corresponding multigraph.
This multigraph should be spanning as we collect all colors in \gi{},
and it is by definition Eulerian. If 
each optimal solution contains at most $2n-2$ edges
(which we have already shown; see~\Cref{lem:num-edges-bound}) and $W$ is constant, then our ILP formulation has~$\Oh(n)$ variables.
Thus (unlike the original instance of \gi)
we can often quickly solve the walk merging problem exactly.\looseness-1

\begin{figure}[t]
  \begin{subfigure}{\textwidth}
    \includegraphics*[width=\linewidth]{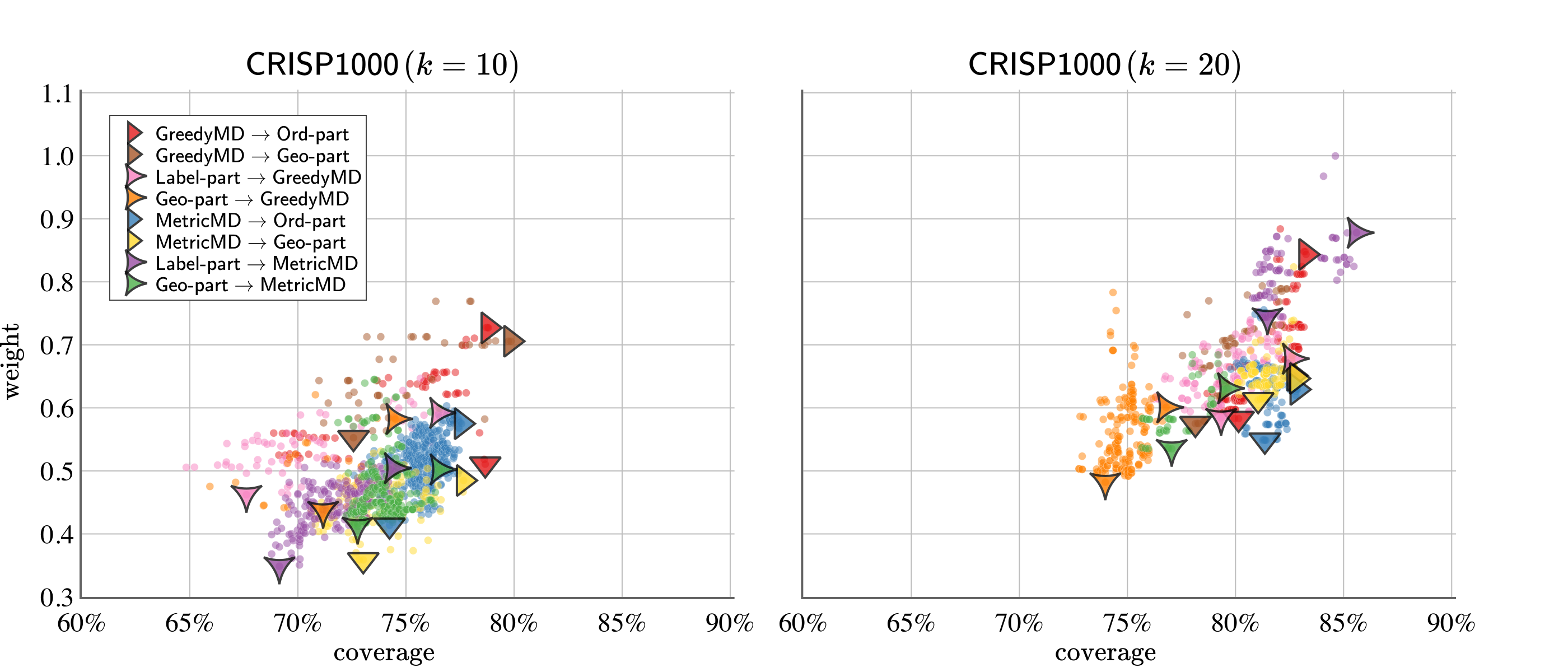}
  \end{subfigure}
  \begin{subfigure}{\textwidth}
    \includegraphics*[width=\linewidth]{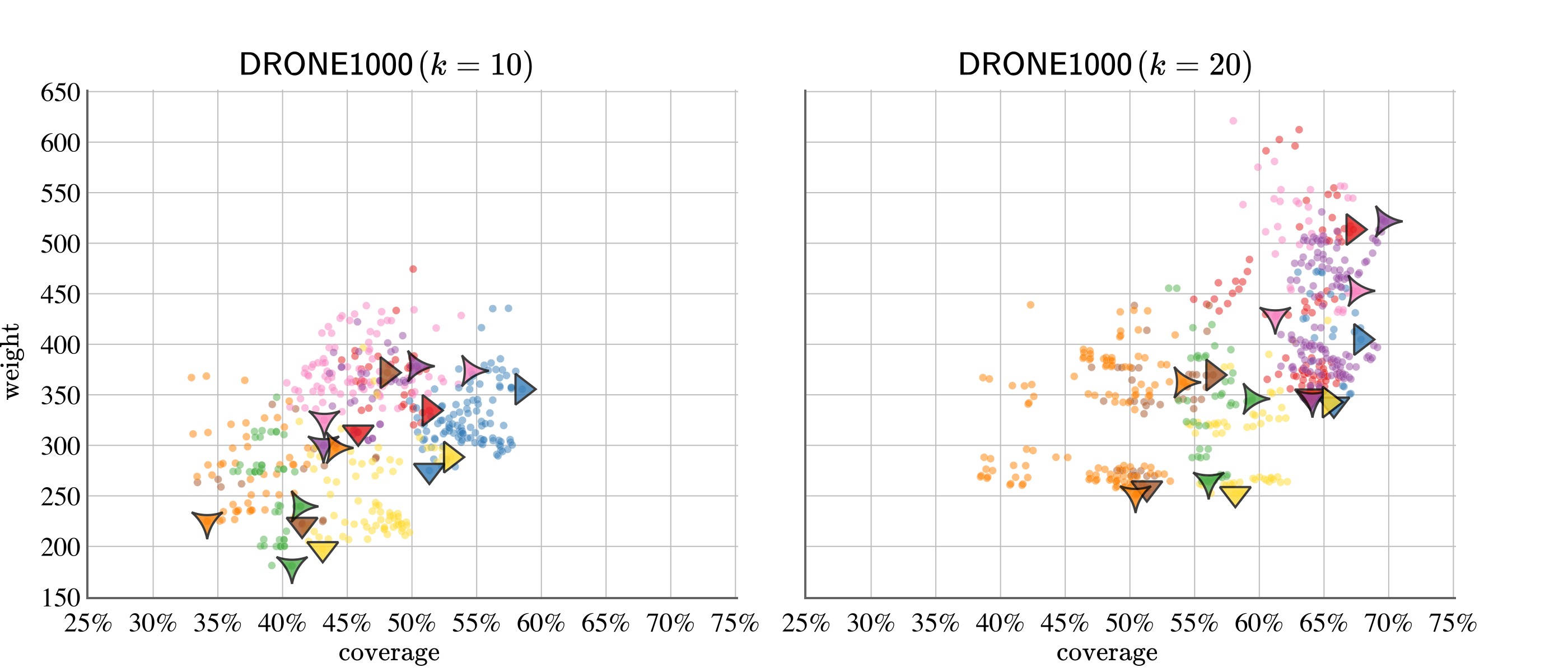}
  \end{subfigure}
  \caption{\label{fig:color-part-main-body} Selected results of color partitioning experiment on datasets \CRISPsmall{} and \DRONEsmall{} with $k \in \{10, 20\}$. Each data point represents a solution computed using \dpipa{} and \exactmerge{}. For each combination of color reduction method and partitioning strategy, we highlight the solutions with maximum coverage (rightward arrow) and with minimum weight (downward arrow).}
  \vspace*{-1em}
\end{figure}

\subsection{Partitioning Colors}\label{sec:partitioning-colors}

Because we want to combine multiple ($W > 1$) walks to form our solution,
it is useful to first partition the colors. This way, each independently
computed walk collects (at least) some disjoint subset of colors. We propose two
algorithms. The first (which we call~\orderedpart, short for \emph{ordered} partitioning) takes an ordered color set $\colorset$
as input and partitions it sequentially, i.e., by selecting the first~$|\colorset|/W$~colors as one
subset, the second~$|\colorset|/W$~colors as another, and so on.
The second (called~\geometricpart, short for \emph{geometric partitioning}) executes \greedymd{} with parameter $W$,
and then partitions $\colorset$ by assigning each color to the most similar (according to the function $f$) of the $W$ selected ``representatives''.\looseness-1

We also tested whether to perform color partitioning \emph{before} or \emph{after} color reduction. In the former
case, the full color set $\colorset$ is partitioned by one of the algorithms described above\footnote{In~\Cref{fig:color-part-main-body}, \orderedpart{} is referred to as \textsf{Label-part} when it is performed before color reduction, to emphasize that in this case the partitioning is based on the (potentially not random) sequence of POI labels given as input.}, and then color reduction is
performed on each subset. In the latter, color reduction is performed to obtain $W\cdot k$ colors, and then these colors are
partitioned into $W$ sets of size $k$ using one of the algorithms described above.
In~\Cref{fig:color-part-main-body} we display the results of our partitioning experiments on
the instances used in \cite{fu2021computationally},
one for a surgical inspection task (\CRISPsmall{})
and another for a bridge inspection task (\DRONEsmall{});
results for extended datasets are deferred to ~\Cref{fig:color-part-appendix}.
We now draw attention to two trends. First, we note that while using \orderedpart{} before color reduction performs well in terms of coverage, particularly for
the larger $k$ values, we believe that this result is confounded somewhat by non-random ordering of the POIs in the input data. That is, we conjecture that the POIs arrive in an order which
conveys some geometric information.
Second, we note that while \metricmd{} seems to outperform \greedymd{} (in terms of coverage) as a color reduction strategy for
\DRONEsmall{} with $k = 10$, this effect is lessened when $k = 20$. We believe that this trend is explainable, as for small $k$ values the greedy procedure may select \emph{only} peripheral POIs,
while a larger $k$ enables good representation of the entire space, including a potentially POI-dense ``core'' of the surface to be inspected. Given the
complexity of the comparative results presented in~\Cref{fig:color-part-main-body}, we
favor the simplest, most generalizable, and most explainable strategy.
For this reason, the experiments of~\Cref{sec:compare-to-iris} are performed using \greedymd{} to reduce colors before partitioning using \orderedpart{}.\looseness-1

\tikzstyle{data_}=[line width=0.5mm,draw=darkgray]
\tikzstyle{task_}=[rounded corners,line width=0.5mm,draw=darkgray]
\tikzstyle{subtask_}=[rounded corners,line width=0.3mm,draw=darkgray]
\tikzstyle{alg_}=[rounded corners=2.5mm,line width=0.3mm,draw=darkgray]
\tikzstyle{process_}=[-{Latex[length=3mm]},line width=0.5mm,draw=darkgray]
\tikzstyle{subprocess_}=[-{Latex[length=2mm]},line width=0.3mm,draw=darkgray]
\tikzstyle{tasktitle_}=[align=left,anchor=west]

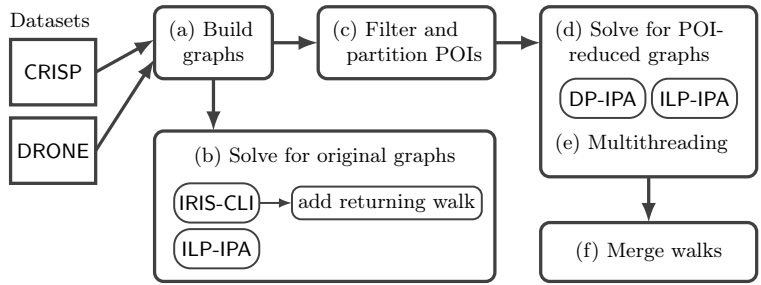
\begin{figure}[t]
  \centering

  \begin{adjustbox}{max width=0.87\textwidth}
  \begin{minipage}[m]{\textwidth}
      \vspace{0pt}
      \centering
      \begin{tikzpicture}
          \draw[data_] (0.5, 2.7) rectangle ++ (1.3, 1.0) node[midway] {\DRONE};
          \draw[data_] (0.5, 3.9) rectangle ++ (1.3, 1.0) node[midway] {\CRISP};
          \node at (1.1, 5.2) {Datasets};

          \draw[process_] (1.8, 3.2) -- (2.7, 4.6);
          \draw[process_] (1.8, 4.4) -- (2.7, 4.9);

          \draw[task_] (2.7, 4.3) rectangle ++ (1.8, 1.1);
          \node[tasktitle_] at (2.8, 4.85) {(a) Build \\\ \ graphs};

          \draw[process_] (3.6, 4.3) -- (3.6, 3.5);

          \draw[task_] (2.7, 1.2) rectangle ++ (5.2, 2.3);
          \node at (5.3, 3.1) {(b) Solve for original graphs};
          \draw[alg_] (3.0, 2.1) rectangle ++ (1.3, 0.6) node[midway] {\iriscli};
          \draw[alg_] (3.0, 1.4) rectangle ++ (1.3, 0.6) node[midway] {\ilpipa};

          \draw[subprocess_] (4.3, 2.4) -- (4.8, 2.4);
          \draw[subtask_] (4.8, 2.15) rectangle ++ (2.9, 0.5) node[midway] {add returning walk};

          \draw[process_] (4.5, 4.85) -- (5.2, 4.85);

          \draw[task_] (5.2, 4.3) rectangle ++ (2.7, 1.1);
          \node[tasktitle_] at (5.3, 4.85) {(c) Filter and \\\ \ partition POIs};
          
          \draw[process_] (7.9, 4.85) -- (8.6, 4.85);
          \draw[task_] (8.6, 2.8) rectangle ++ (3.3, 2.6);
          \node[tasktitle_] at (8.7, 4.85) {(d) Solve for POI-\\\ \ reduced graphs};
          \node[tasktitle_] at (8.7, 3.3) {(e) Multithreading};
          \draw[alg_] (8.9, 3.7) rectangle ++ (1.3, 0.6) node[midway] {\dpipa};
          \draw[alg_] (10.3, 3.7) rectangle ++ (1.3, 0.6) node[midway] {\ilpipa};

          \draw[task_] (8.6, 1.2) rectangle ++ (3.3, 0.9) node[midway]{(f) Merge walks};
          \draw[process_] (10.25, 2.8) -- (10.25, 2.1);
      \end{tikzpicture}
  \end{minipage}
  \end{adjustbox}
  \caption{%
  Overview of our experiment pipeline.
  }
  \label{fig:pipeline}
  \vspace*{-1em}
\end{figure}

\section{Empirical Evaluation}\label{sec:experiments}

To assess the practicality of our proposed algorithms, we ran extensive experiments
on a superset of the real-world instances used in \cite{fu2021computationally}.
\Cref{fig:pipeline} shows an overview of the experiment pipeline.
We first built RRGs using \iriscli{}, originating from the \CRISP{} and \DRONE{} datasets (\Cref{fig:pipeline}~(a)).
We tested \iriscli{} and \ilpipa{} on these instances with no additional color reduction.
As \iriscli{} iteratively outputs an $s$-$t$ walk for some vertex $t \in V(G)$,
we completed each walk with the shortest $t$-$s$ path\footnote{%
\label{fn-neg}The time taken for augmenting and merging walks was negligible.
} to ensure a fair comparison while still giving as much credit as possible to \iriscli{} (\Cref{fig:pipeline}~(b)).
For \dpipa{}, we filter and partition POIs to obtain $3$ sets of $k$ POIs,
where $k=10,20$ (\Cref{fig:pipeline}~(c)).
Then, we ran \dpipa{} to exactly solve \gi{} for POI-reduced instances.
In addition, we ran \ilpipa{} for comparing color reduction/partitioning algorithms (\Cref{fig:pipeline}~(d))
and measured speedups of those algorithms with different number of threads (\Cref{fig:pipeline}~(e)).
Lastly, we merged the walks using our algorithms to construct a ``combined'' closed walk (\Cref{fig:pipeline}~(f)).
Here we define the ``search time'' for the combined walk as the total of the search times of single-run walks
plus the time taken for merging walks\footref{fn-neg}.
Except in experiment (e), we set the time limit of each algorithm to 900 seconds (15 minutes), and 
used $80$ threads for \dpipa{} and \ilpipa{}.\looseness-1

We tested on four \gi{} instances, two of which replicate the instances used in \cite{fu2021computationally}.
The first dataset, \CRISP, is a simulation for medical inspection tasks of the Continuum Reconfigurable Incisionless Surgical Parallel (CRISP)
robot \cite{anderson2017continuum,mahoney2016reconfigurable}.
The dataset simulates a scenario segmented from a CT scan of a real patient
with a pleural effusion---a serious medical condition that can cause the collapse of a patient's lung.
The second, \DRONE{}, is an infrastructure inspection scenario,
in which a UAV with a camera is
tasked with inspecting the critical structural features of a bridge.
Its inspection points are the surface vertices in the 3D mesh model of a bridge structure used in \cite{fu2021computationally}.
To match the experiments in \cite{fu2021computationally}, we used \iriscli{} to build RRGs with $n_{\text{build}}=1000$
and, for \CRISP{}, uniformly randomly selected $4200$ POIs. We call these instances \CRISPsmall{} and \DRONEsmall{}.
Also, for each dataset, we built RRGs with $n_{\text{build}}=2000$ (denoted \CRISPbig{} and \DRONEbig{}).
\Cref{appendix:experiments} details our graph instances
and experiment environment.
Code and data to replicate all experiments are available at \cite{robotic-brewing}.\looseness-1

\begin{figure}[t]
  \begin{subfigure}{\textwidth}
    \includegraphics[width=\linewidth]{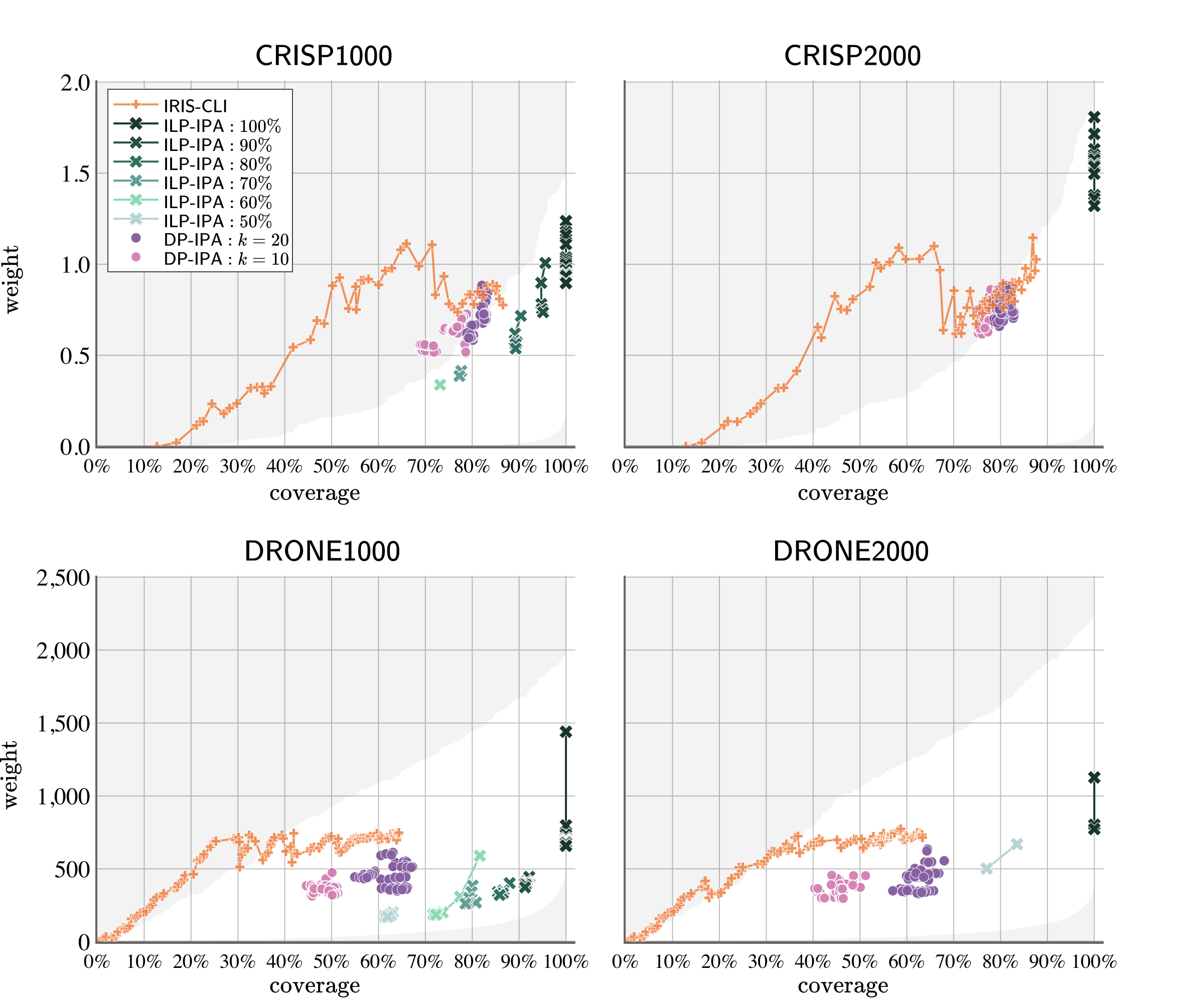}
  \end{subfigure}
  \caption{\label{fig:headline} Performance of \iriscli{}, \ilpipa{} and
  \dpipa{} (with \greedymd) on \DRONE~and \CRISP~benchmarks. Each data point represents a computed inspection plan; coverage is shown as a percentage of all POIs in the input graph. The area shaded in gray is outside the upper/lower bounds given in Section~\ref{sec:theory-bounds}.}
  \vspace*{-1em}
\end{figure}

\subsubsection*{Comparison to \iriscli{}.}\label{sec:compare-to-iris}

First we compare the overall performance of our proposed algorithms to that of \iriscli{}.
In this experiment, we ran \iriscli{} with all original instances,
\ilpipa{} with all original instances and $t=\frac{i}{10} \cdot |\colorset|$ for $5 \leq i \leq 10$,
and \dpipa{} with POI-reduced instances accompanied by walk-merging strategies,
\greedymd{} (\metricmd{} in Appendix \Cref{fig:headline2}),
\orderedpart{} \emph{after} color reduction, and \exactmerge{} with $k \in \{10,20\}$.
We additionally computed the upper and lower bounds from~\Cref{sec:theory-bounds} for all possible $k$ values.

\Cref{fig:headline} plots the coverage and weight of each solution obtained within the time limit.
\iriscli{} achieved around $87\%$ coverage on both \CRISP{} instances.
\ilpipa{} outperformed \iriscli{} on \CRISPsmall{} by providing (i) for $t=0.8\cdot |\colorset|$, slightly better coverage
paired with a $30\%$ reduction in weight, and (ii) for $t=|\colorset|$, perfect coverage with only a $16\%$ increase in weight.
On \CRISPbig{}, \ilpipa{} failed to find a solution except with $t=|\colorset|$.
Meanwhile, \dpipa{} was competitive with \iriscli{}, finding walks with moderate reductions in weight at the expense of slightly reduced coverage ($83\%$).
The differences between \iriscli{} and our algorithms are more significant on \DRONE{},
where \dpipa{} (with $k=20$) outperformed \iriscli{} by providing more coverage ($68\%$ vs. $64\%$)
while reducing weight by over $50\%$. \ilpipa{} outperformed \iriscli{} by even larger margins on \DRONEsmall{},
but did not produce many solutions within the time limit on \DRONEbig{}.\looseness-1

To summarize, \ilpipa{} is the most successful on smaller instances, and works with various values of $t$.
With larger graphs, \ilpipa{} is more likely to time out when $t < |\colorset|$ (as described in \Cref{sec:ilp-theory}, the ILP formulation in this case is more involved).
\dpipa{} is more robust on larger instances, outperforming \iriscli{} in terms of solution weight while providing similar coverage.

\begin{figure}[t]
  \begin{subfigure}{\textwidth}
    \includegraphics[width=\linewidth]{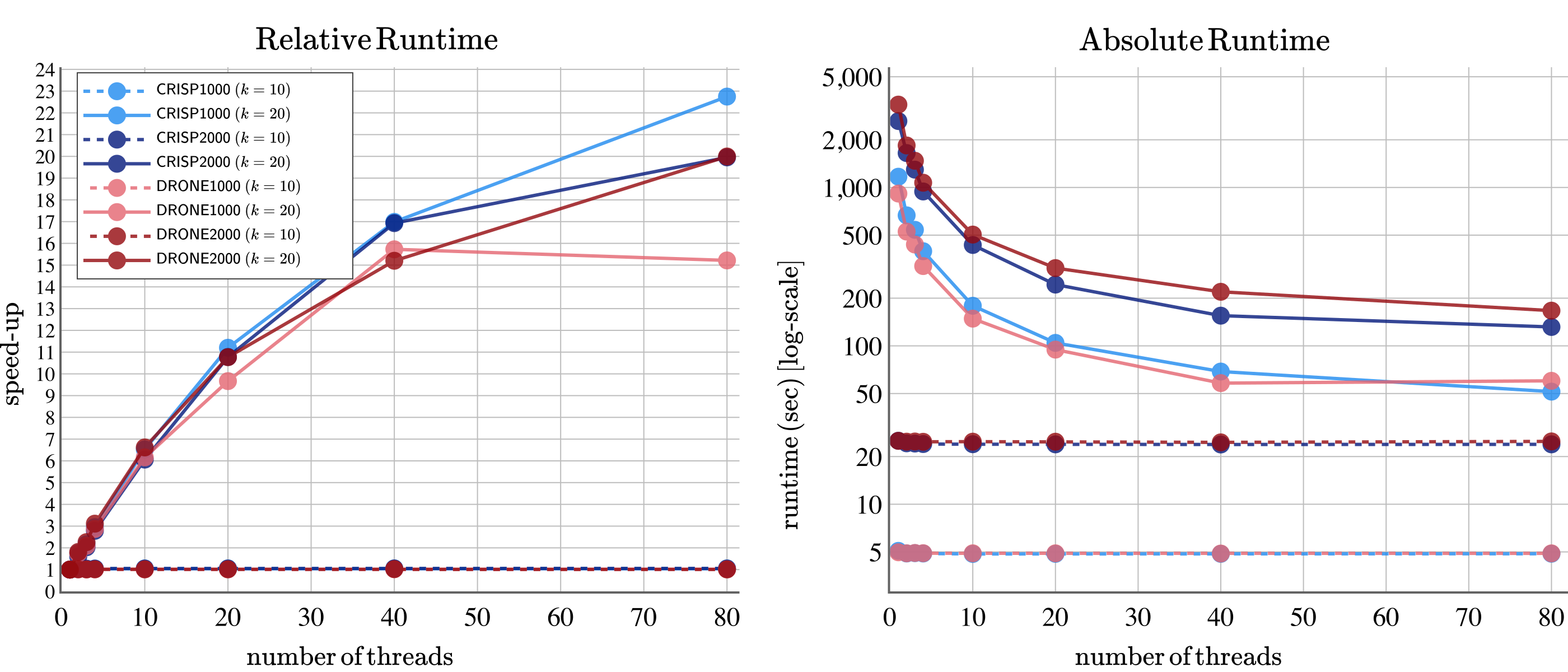}
  \end{subfigure}
  \caption{\label{fig:dp-multithreading} Relative (left) and absolute (right) runtimes for \dpipa{} with 1-80 threads.
  The relative runtime (speed-up) is with respect to a single thread.}
\end{figure}

\subsubsection*{Upper and lower bounds.}
First, we observe that the curvatures of upper bounds (recall~\Cref{sec:theory-bounds}) are quite different in \CRISP{} and \DRONE{}.
We believe the geometric distribution of POIs explains this difference;
with \CRISP, the majority of POIs are close to the POIs seen at the starting point,
which leads to concave upper-bound curves. \DRONE{}, on the other hand, exhibits a linear trend in upper bounds
because POIs are more evenly distributed in the 3D space, and the obtained solutions are far from these bounds.
On the lower bound side, we obtain little insight on \CRISP{} but see that on \DRONE{}, it allows 
us to get meaningful bounds on the ratio of our solution's weight to that of an optimal inspection plan. 
For example, the lower bounds with $t=|\colorset|$ for \DRONEsmall{} and \DRONEbig{}
are $466.74$ and $364.72$, respectively.
The best weights by \ilpipa{} are $658.35$ and $773.39$, giving approximation ratios of
$1.4$ and $2.1$ (respectively).

\subsubsection*{Multithreading Analysis.} Our implementations of \dpipa{} and \ilpipa{} both allow for multithreading, but \iriscli{} cannot be parallelized without
extensive modification (i.e., the algorithm is inherently sequential). 
\Cref{fig:dp-multithreading} illustrates order-of-magnitude runtime improvements for \dpipa{} when using multiple cores. Analogous results for \ilpipa{} are deferred to~\Cref{appendix:multithreading}.

\subsubsection*{Empirical Analysis of Walk-Merging Algorithms.}

The results reported in~\Cref{fig:headline} are all computed using \exactmerge{} for walk merging. We also evaluated the effectiveness of \concatmerge{} and \greedymerge{} in terms of both runtime and resulting (merged) walk weight. We observed that while \greedymerge{} is a heuristic, it produces walks of nearly optimal weight with negligible runtime increase as compared to \concatmerge{}. Meanwhile, the runtime of \exactmerge{} was always within a factor of two of \concatmerge{}; given the small absolute runtimes (<0.1 seconds in all cases), we chose to proceed with \exactmerge{} to minimize the weight of the merged walk. Complete experimental results are shown in~\Cref{appendix:experiments}.

\subsubsection*{Comparing \ilpipa{} and \dpipa{}.}\label{sec:compare-ilp-dp}
  In this work we have contributed two new \gi{} solvers, namely \dpipa{} and \ilpipa{}. We conclude this section with a
  brief discussion of their comparative strengths and weaknesses. As discussed previously, for small
  graphs (e.g., \CRISPsmall{} and \DRONEsmall{}) \ilpipa{} is clearly the best option, as it provides higher coverage with
  less weight. However, \dpipa{} performs better when the graph is large. In particular, if in some application minimizing weight is more important
  than achieving perfect coverage, then \dpipa{} is preferable in large graphs. One might ask whether this trade-off (between weight and coverage) can also be tuned
  for \ilpipa{} by setting $t < |\colorset|$, but we emphasize that in practice this choice significantly
  increases the runtime of \ilpipa{}, such that it is impractical on large graphs.
  This is clear from the data presented in~\Cref{fig:headline},
  and is also detailed in~\Cref{appendix:additional}.
  We observe that \ilpipa{} becomes less competitive with \dpipa{} as $n$ and $k$ grow.

\section{Conclusion}\label{sec:conclusion}
In this work, we took tangible and meaningful steps toward mapping the \gi{} planning problem in robotics to established problems (e.g. \prob{Generalized TSP}).
We presented two algorithms, \dpipa{} and \ilpipa{}, to solve the problem under this abstraction, based on dynamic programming and integer linear programming.
We presented multiple strategies for leveraging these algorithms on relevant robotics examples lending insight into the choices that can be made to use these methods in emerging problems.
We then evaluated these methods and strategies on two complex robotics applications, outperforming the state of the art.\looseness-1

Our approach of creating several reduced color sets and merging walks offers a new paradigm for leveraging algorithms whose complexity has high dependence on the number of POIs, and opens the door for future exploration.
We plan to see how these methods perform and scale with more than three walks.
Further, it remains to implement these algorithms on real-world, physical robots and inspection tasks.\looseness-1

\section*{Acknowledgements}
This work was supported in part by the European Research Council (ERC) under the European Union's Horizon 2020
research and innovation programme (grant agreement No. 819416), by
the Gordon \& Betty Moore Foundation's Data Driven Discovery Initiative
(award GBMF4560 to Blair D. Sullivan), and by the National Science Foundation (award ECCS-2323096 to Alan Kuntz).

\bibliographystyle{splncs04}
\bibliography{refs,Kuntz-Refs}

\begin{thebibliography}{10}
\providecommand{\url}[1]{\texttt{#1}}
\providecommand{\urlprefix}{URL }
\providecommand{\doi}[1]{https://doi.org/#1}

\bibitem{anderson2017continuum}
Anderson, P.L., Mahoney, A.W., Webster, R.J.: Continuum reconfigurable parallel robots for surgery: Shape sensing and state estimation with uncertainty. IEEE robotics and automation letters  \textbf{2}(3),  1617--1624 (2017)

\bibitem{applegate2006traveling}
Applegate, D.L., Bixby, R.E., Chvatál, V., Cook, W.J.: The {Traveling} {Salesman} {Problem}: {A} {Computational} {Study}, vol.~17. Princeton University Press (2006)

\bibitem{atkar2005uniform}
Atkar, P.N., Greenfield, A., Conner, D.C., Choset, H., Rizzi, A.A.: Uniform coverage of automotive surface patches. The International Journal of Robotics Research  \textbf{24}(11),  883--898 (2005)

\bibitem{bjorklund2012shortest}
Björklund, A., Husfeldt, T., Taslaman, N.: Shortest {Cycle} {Through} {Specified} {Elements}. In: Proceedings of the {Twenty}-{Third} {Annual} {ACM}-{SIAM} {Symposium} on {Discrete} {Algorithms}. pp. 1747--1753. Society for Industrial and Applied Mathematics (Jan 2012)

\bibitem{cheng2008time}
Cheng, P., Keller, J., Kumar, V.: Time-optimal {UAV} trajectory planning for {3D} urban structure coverage. In: 2008 IEEE/RSJ International Conference on Intelligent Robots and Systems. pp. 2750--2757 (2008)

\bibitem{Cho2021_ISMR}
Cho, B.Y., Hermans, T., Kuntz, A.: Planning {Sensing} {Sequences} for {Subsurface} {3D} {Tumor} {Mapping}. In: 2021 {International} {Symposium} on {Medical} {Robotics} ({ISMR}). pp.~1--7 (Nov 2021)

\bibitem{Cho2024_ICRA}
Cho, B.Y., Kuntz, A.: Efficient and {Accurate} {Mapping} of {Subsurface} {Anatomy} via {Online} {Trajectory} {Optimization} for {Robot} {Assisted} {Surgery}. In: {IEEE} {International} {Conference} on {Robotics} and {Automation} ({ICRA}). pp. 15478--15484 (May 2024)

\bibitem{cohen2021tropical}
Cohen, J., Italiano, G.F., Manoussakis, Y., Thang, N.K., Pham, H.P.: Tropical paths in vertex-colored graphs. Journal of Combinatorial Optimization  \textbf{42}(3),  476--498 (Oct 2021)

\bibitem{cohen2019severalgraph}
Cohen, N.: Several {Graph} problems and their {Linear} {Program} formulations. Research report, INRIA (2019)

\bibitem{couetoux2017maximum}
Couëtoux, B., Nakache, E., Vaxès, Y.: The {Maximum} {Labeled} {Path} {Problem}. Algorithmica  \textbf{78}(1),  298--318 (May 2017)

\bibitem{cygan2015parameterized}
Cygan, M., Fomin, F.V., Kowalik, {\L}., Lokshtanov, D., Marx, D., Pilipczuk, M., Pilipczuk, M., Saurabh, S.: Parameterized algorithms, vol.~4. Springer (2015)

\bibitem{diestel2005graph}
Diestel, R.: Graph theory, Graduate {Texts} in {Mathematics}, vol.~173. Springer-Verlag, Berlin, third edn. (2005)

\bibitem{dimitrijevic1997anefficient}
Dimitrijević, V., Šarić, Z.: An efficient transformation of the generalized traveling salesman problem into the traveling salesman problem on digraphs. Information Sciences  \textbf{102}(1),  105--110 (1997)

\bibitem{englot2017planning}
Englot, B., Hover, F.: Planning complex inspection tasks using redundant roadmaps. In: Robotics Research : The 15th International Symposium ISRR. pp. 327--343. Springer International Publishing (2017)

\bibitem{euler1741solutio}
Euler, L.: Solutio problematis ad geometriam situs pertinentis. Commentarii academiae scientiarum Petropolitanae pp. 128--140 (1741)

\bibitem{fomin2023fixed}
Fomin, F.V., Golovach, P.A., Korhonen, T., Simonov, K., Stamoulis, G.: Fixed-{Parameter} {Tractability} of {Maximum} {Colored} {Path} and {Beyond}. In: Proceedings of the 2023 {Annual} {ACM}-{SIAM} {Symposium} on {Discrete} {Algorithms} ({SODA}). pp. 3700--3712. Society for Industrial and Applied Mathematics (Jan 2023)

\bibitem{fu2021computationally}
Fu, M., Salzman, O., Alterovitz, R.: Computationally-efficient roadmap-based inspection planning via incremental lazy search. In: 2021 IEEE International Conference on Robotics and Automation (ICRA). pp. 7449--7456. IEEE (2021)

\bibitem{gonzalez1985clustering}
Gonzalez, T.F.: Clustering to minimize the maximum intercluster distance. Theoretical computer science  \textbf{38},  293--306 (1985)

\bibitem{Halperin2017_Book}
Halperin, D., Salzman, O., Sharir, M.: Algorithmic {Motion} {Planning}. In: Handbook of {Discrete} and {Computational} {Geometry}. Chapman and Hall/CRC, 3 edn. (2017)

\bibitem{Harris1995_Chest}
Harris, R.J., Kavuru, M.S., Mehta, A.C., Medendorp, S.V., Wiedemann, H.P., Kirby, T.J., Bice, T.W.: The impact of thoracoscopy on the management of pleural disease. Chest  \textbf{107}(3),  845--852 (1995)

\bibitem{hierholzer1873ueber}
Hierholzer, C., Wiener, C.: Ueber die {Möglichkeit}, einen {Linienzug} ohne {Wiederholung} und ohne {Unterbrechung} zu umfahren. Mathematische Annalen  \textbf{6}(1),  30--32 (1873)

\bibitem{karaman2011sampling}
Karaman, S., Frazzoli, E.: Sampling-based algorithms for optimal motion planning. The international journal of robotics research  \textbf{30}(7),  846--894 (2011)

\bibitem{kou1981fast}
Kou, L., Markowsky, G., Berman, L.: A fast algorithm for {Steiner} trees. Acta informatica  \textbf{15},  141--145 (1981)

\bibitem{lien1993tranformation}
Lien, Y.N., Ma, E., Wah, B.W.S.: Transformation of the generalized traveling-salesman problem into the standard traveling-salesman problem. Information Sciences  \textbf{74}(1),  177--189 (1993)

\bibitem{Light2007_Pleural}
Light, R.W.: Pleural diseases. Lippincott Williams \& Wilkins (2007)

\bibitem{mahoney2016reconfigurable}
Mahoney, A.W., Anderson, P.L., Swaney, P.J., Maldonado, F., Webster, R.J.: Reconfigurable parallel continuum robots for incisionless surgery. In: 2016 IEEE/RSJ International Conference on Intelligent Robots and Systems (IROS). pp. 4330--4336. IEEE (2016)

\bibitem{robotic-brewing}
Mizutani, Y., Salomao, D.C., Crane, A., Bentert, M., Drange, P.G., Reidl, F., Kuntz, A., Sullivan, B.D.: Accompanying source code. \url{https://github.com/TheoryInPractice/robotic-brewing/}

\bibitem{Noppen2010_SRCCM}
Noppen, M.: The utility of thoracoscopy in the diagnosis and management of pleural disease. In: Seminars in respiratory and critical care medicine. vol.~31, pp. 751--759 (2010), tex.number: 06 tex.organization: Thieme Medical Publishers

\bibitem{pop2024comprehensive}
Pop, P.C., Cosma, O., Sabo, C., Sitar, C.P.: A comprehensive survey on the generalized traveling salesman problem. European Journal of Operational Research  \textbf{314}(3),  819--835 (2024)

\bibitem{rice2012exact}
Rice, M.N., Tsotras, V.J.: Exact graph search algorithms for generalized traveling salesman path problems. In: Klasing, R. (ed.) Experimental Algorithms. pp. 344--355. Springer Berlin Heidelberg, Berlin, Heidelberg (2012)

\bibitem{rice2013parameterizedalgorithms}
Rice, M.N., Tsotras, V.J.: Parameterized algorithms for generalized traveling salesman problems in road networks. In: Proceedings of the 21st {ACM} {SIGSPATIAL} {International} {Conference} on {Advances} in {Geographic} {Information} {Systems}. pp. 114--123. {SIGSPATIAL}'13, Association for Computing Machinery, New York, NY, USA (Nov 2013)

\bibitem{safra2006complexity}
Safra, S., Schwartz, O.: On the complexity of approximating {TSP} with neighborhoods and related problems. computational complexity  \textbf{14}(4),  281--307 (Mar 2006)

\end{thebibliography}

\newpage
\appendix
\section{Walk Merging: Optimality in the Limit}
\label{appendix:B}
We argue that our strategy for merging walks is a simplification of an algorithm that is optimal in the limit, given sufficient runtime.
Note that the dynamic program behind \Cref{thm:DP} is optimal (always produces a walk of minimum weight that collects the given colors).
A very simple strategy that is also optimal is to select an arbitrary permutation~$(c_1,c_2,\ldots,c_k)$ of the colors and then compute a shortest walk that collects all colors in the respective order. If we repeat this for each possible permutation, then at some point, we will find an optimal solution.

We now observe that computing a walk that collects all colors in the guessed order can be computed in polynomial time by the following dynamic program~$T$ that stores for each vertex~$v$ and each integer~$i \in [k]$ the length of a shortest walk between~$s$ and~$v$ that collects the first~$i$ colors in the guessed order. Therein, we use a second table~$D$.
\begin{align*}
	D[v,i] &= \begin{cases} T[u,i-1] + \dist(u,v) & \text{if } c_i \in \col(v) \\ \infty & \text{else} \end{cases}\\
	T[v,i] &= \min_{u \in V} D[u,i] + \dist(u,v)
\end{align*}
We mention that we assume that~$\dist(v,v)=0$ for each vertex~$v$.

We now modify the above strategy to achieve a better success probability than when permutations are chosen randomly.
Instead of guessing the entire sequence of colors, we guess \emph{buckets} of colors, that is, a sequence of~$c$ sets for some integer~$c$ that form a partition of the set of colors into sets of size~$\nicefrac{k}{c}$ (appropriately rounded, for this presentation, we will assume that~$k$ is a multiple of~$c$).
Note that there are~$\nicefrac{k!}{((\nicefrac{k}{c})!)^c}$ possible guesses for such buckets.
For each bucket, let~$S_i$ be the set of colors in the bucket.
We can now compute for each pair~$u,v$ of vertices a shortest walk between~$u$ and~$v$ that collects all colors in~$S_i$ by running the dynamic program behind \Cref{thm:DP} for color set~$S_i$ from all vertices~$u \in V$.
Let this computed value be~$S[u,v,S_i]$.
Given a guess for a sequence of buckets, we can now compute an optimal solution corresponding to this guess by modifying the above dynamic program as follows.
\begin{align*}
	T'[v,1] &= S[s,v,S_1]\\
	T'[v,i] &= \min_{u \in V} T[u,i-1] + S[u,v,S_i] \text{ if } i > 1
\end{align*}
Note that~$T'[v,i] \leq T[v,ci]$ if the sequence used to compute~$T$ corresponds to the set of buckets used to compute~$T'$.
This algorithm again achieves optimality in the limit (that is, given enough runtime, it will find a minimum weight walk).

In order to avoid computing~$S$ for all pairs of vertices, we decided to implement a simplification where we only compute~$S'[S_i] = S[s,s,S_i]$.
This version is not guaranteed to find an optimal solution like the full DP above, as there are examples where no optimal solution returns back to~$s$ before the very end. As an example, consider a star graph where~$s$ is a leaf and each other leaf has a unique color.
However, this simplification is faster by a factor of~$n$ and uses a factor of~$n$ less memory while performing well in practice.

\section{Color Reduction}
\label{appendix:color-reduction}

In this section, we describe our color reduction strategies in more detail and provide
empirical data comparing their effectiveness.

\subsubsection{Initializing \greedymd{}.} The Gonzalez~\cite{gonzalez1985clustering} algorithm on which
\greedymd{} is based requires that the set $\colorset'$ be initialized with at least one color. The simplest
strategy is to simply add a uniformly randomly selected color to $\colorset'$ and then proceed with the algorithm.
In practice, we found it more effective to set~$\colorset' = \col(s) \neq \emptyset$ (see Algorithm~\ref{alg:greedyMD}).
That is, we begin by adding some color which is visible from the source vertex $s$. We then greedily add
$k$ more colors. At the end of the algorithm, we return $\colorset' \setminus \col(s)$. Intuitively, the
justification for this approach is that we collect the colors $\col(s)$ ``for free'' since every solution walk
begins at $s$. Consequently, we do not need to ensure that $\colorset'$ is representative of these colors, or of
colors which are very similar to them. Empirically, we found that this initialization strategy significantly
improved the coverage of the resulting solutions. 

\subsubsection{Introducing \metricmd{} and \outliermd{}.} A potential shortcoming of
\greedymd{} is that it favors outlier colors. That is, because at each iteration it chooses the color which
is most dissimilar to the previously selected colors, we can be sure that outlier colors which are very dissimilar to
every other color will be selected. This may be undesirable for two reasons. First, if the similarity function $f$ is correlated
to colors being visible from the same vertices, then discarding outliers from $\colorset'$ may improve coverage (as computed on $\colorset$).
Second, if the similarity function $f(c_1, c_2)$ is correlated to the shortest distance between vertices labeled with $c_1$ and $c_2$, then
discarding outliers from $\colorset'$ may reduce the weight needed for a walk collecting all colors in $\colorset'$.

We designed and tested two strategies to mitigate these effects. The first, \outliermd{} (see Algorithm~\ref{alg:outlierMD}), uses an additional
scaling parameter $r \geq 1$. \greedymd{} is used to form a representative color set $\colorset'$ of size $rk$.
Next, $\colorset'$ is partitioned into $rk$ clusters by assigning each color $c \in \colorset$ to a cluster
uniquely associated with the representative $c' \in \colorset'$ to which $c$ is most similar. Finally, we return the
$k$ colors in $\colorset'$ associated with the largest clusters.

\begin{figure}[t]
  \input{figures/outlierMDpseudocode.tex}
\end{figure}

The second strategy, \metricmd{}, assumes that our colors are embeddable in a metric space. This is true, for example, when
colors represent positions in $\R^3$ on some surface mesh of the object to be inspected. In this case, we
begin by using \greedymd{} to find a representative colors set $\colorset'$ of size $k$. Next, we
perform $k$-means clustering on $\colorset$, using $\colorset'$ as the initial centroids.
The resulting centroids are positions in space, and may not perfectly match the positions of any colors. To deal with this,
we simply choose the colors closest to the centroids, and return these as our representative color set.
See Algorithm~\ref{alg:metricMD}.

\begin{figure}[t]
  \input{figures/metricMDpseudocode.tex}
\end{figure}

\subsubsection{Empirical Evaluation of Color Reduction Schemes.}

\begin{figure}
  \centering
  \begin{subfigure}{0.9\textwidth}
    \includegraphics[width=\linewidth]{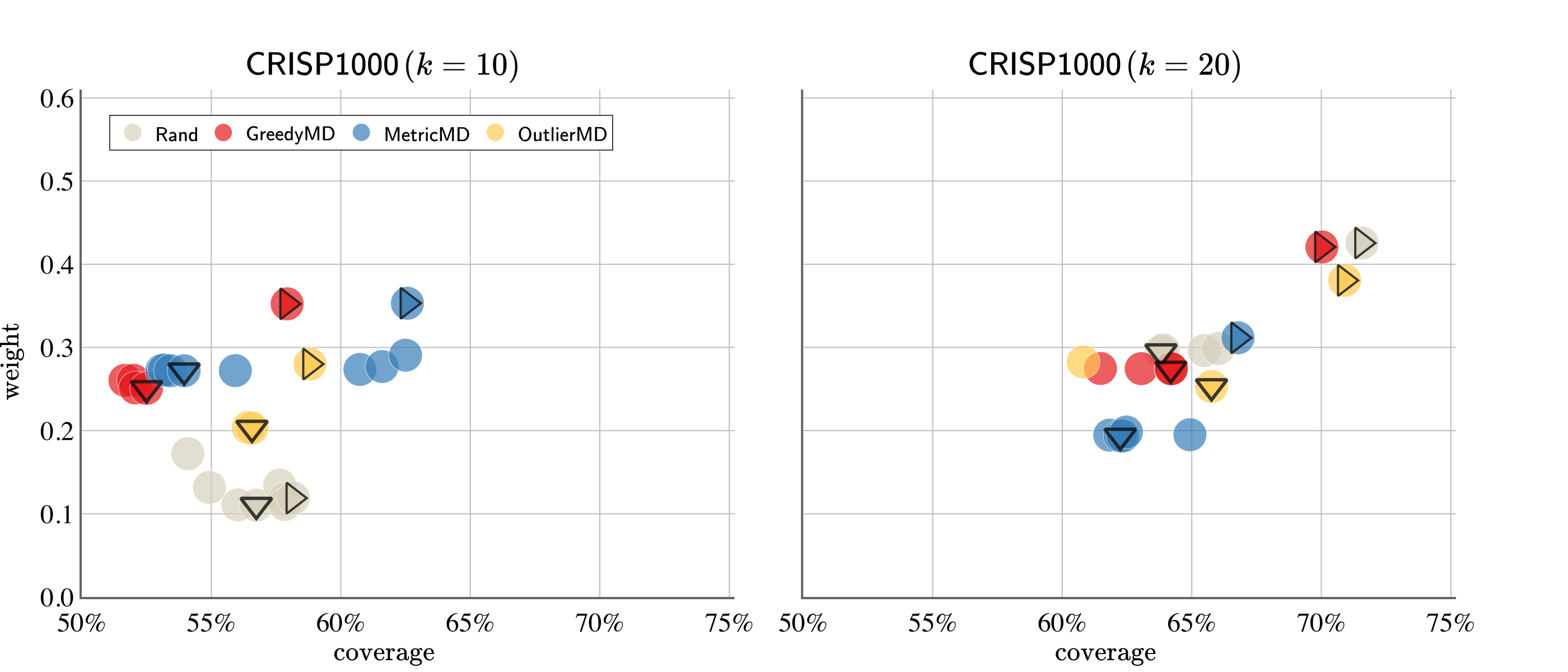}
  \end{subfigure}
  \begin{subfigure}{0.9\textwidth}
    \includegraphics[width=\linewidth]{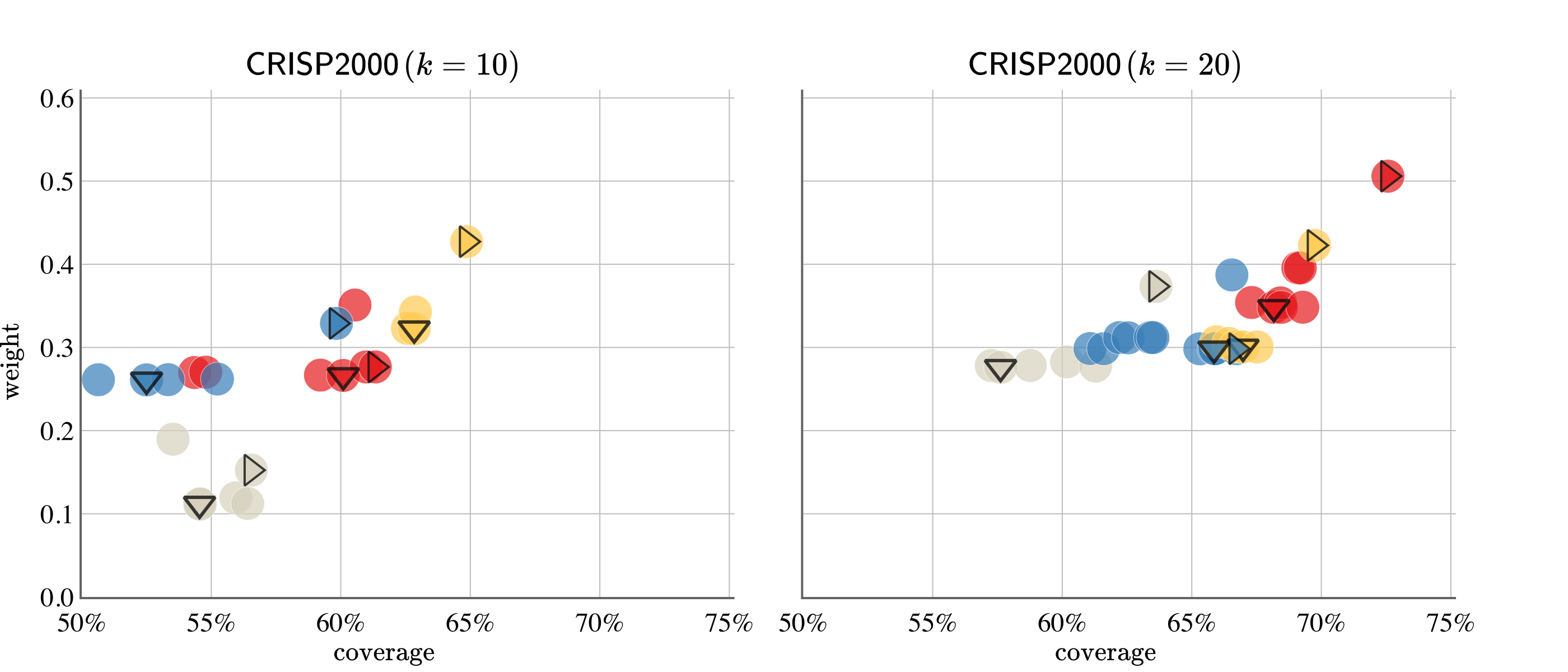}
  \end{subfigure}
  \begin{subfigure}{0.9\textwidth}
    \includegraphics[width=\linewidth]{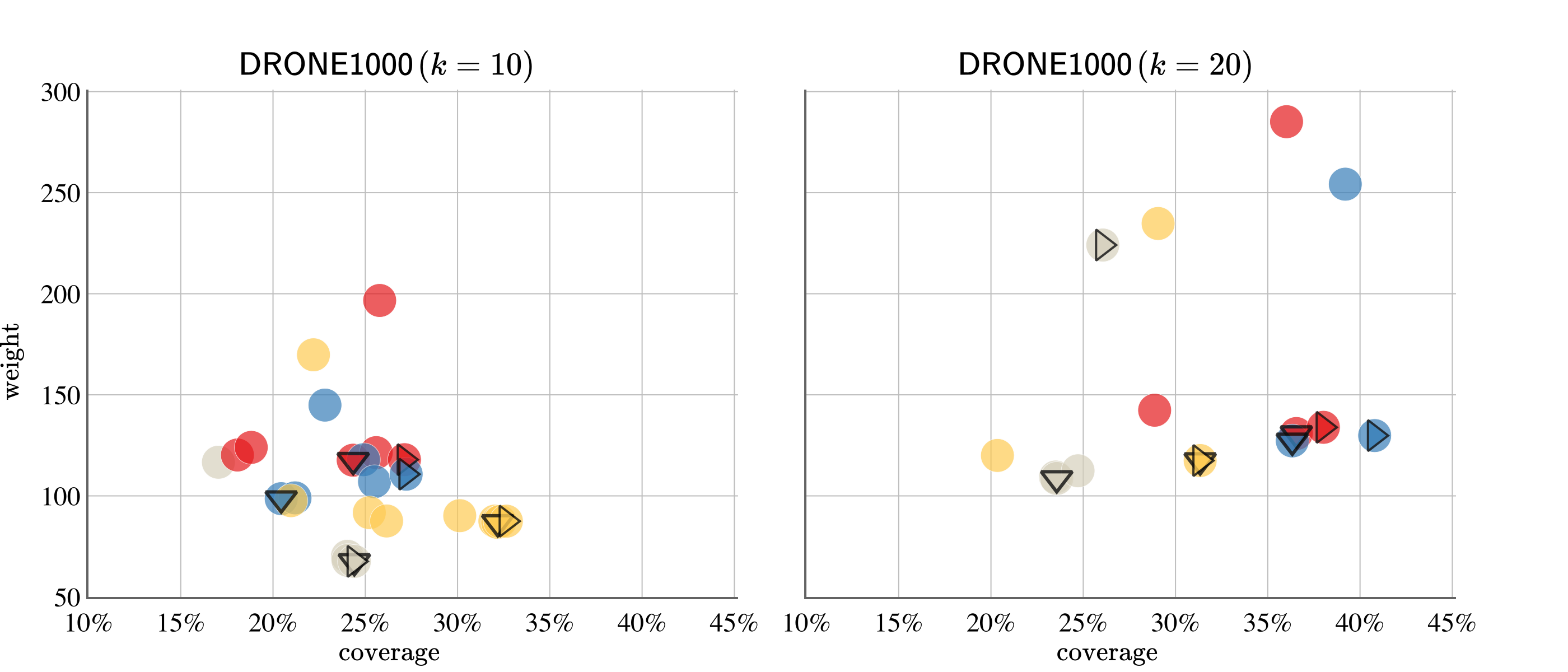}
  \end{subfigure}
  \begin{subfigure}{0.9\textwidth}
    \includegraphics[width=\linewidth]{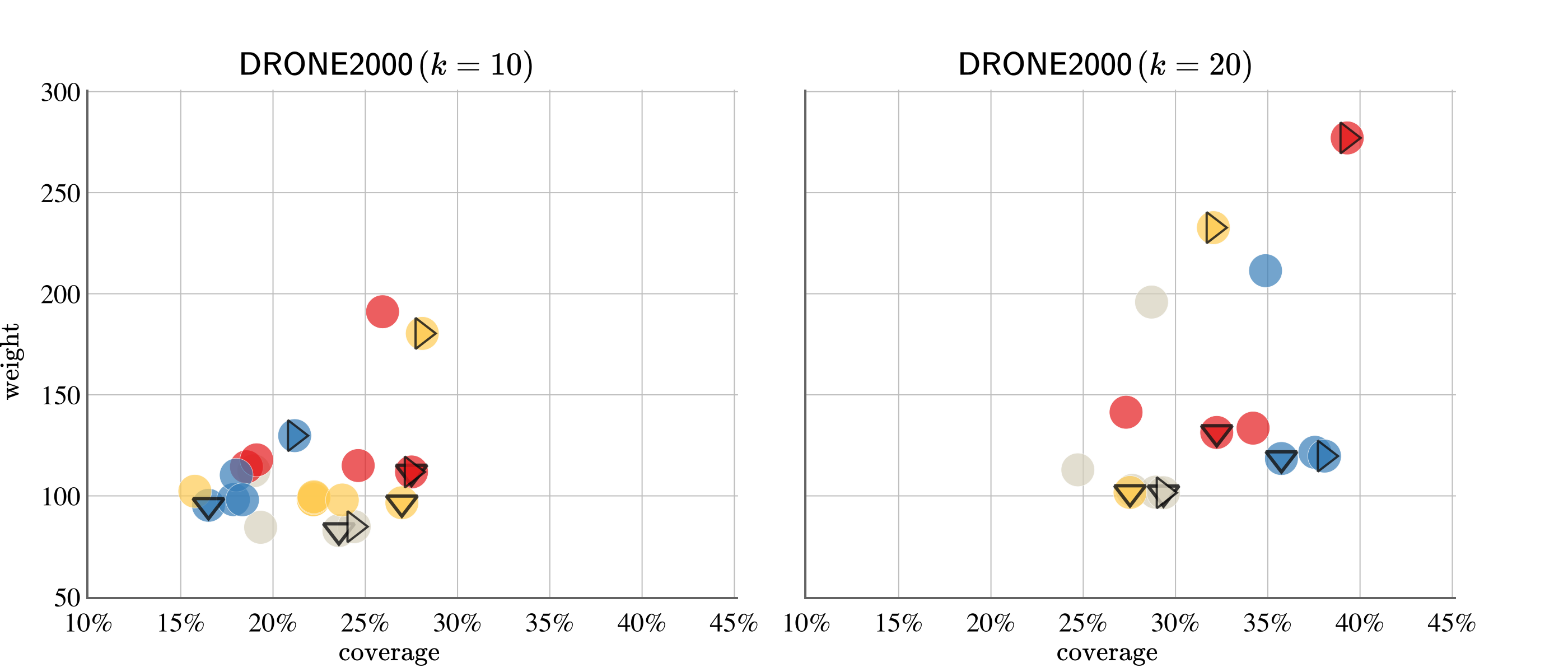}
  \end{subfigure}
  \caption{\label{fig:color-reduction-data} Results of color reduction experiments for \randmd{}, \greedymd{}, \outliermd{}, and \metricmd{}.
  Each datapoint represents a \gi{} solution generated by \dpipa{}. For each algorithm, the solution with the highest coverage (computed in the original, non-color-reduced graph) is marked with a rightward-pointing arrow, and the minimum weight solution is marked with a downward-pointing arrow.\looseness-1}
\end{figure}

We experimentally evaluated our color reduction schemes (\greedymd{}, \outliermd{}, and \metricmd{}) along with the baseline \randmd{}
on each of our four datasets, with $k \in \{10, 20\}$. To perform the evaluation, each algorithm was used to create a representative set
of $k$ colors, and then \dpipa{} was used to solve \gi{} on the color-reduced graphs. We chose \dpipa{} rather than \ilpipa{} for comparison because
the former is the solver which needs color reduction to compute walks on our datasets (i.e., \ilpipa{} can run on the
\DRONE{} and \CRISP{} datasets without any color reduction). 
The results of these experiments are displayed in~\Cref{fig:color-reduction-data}. In each plot
displayed in~\Cref{fig:color-reduction-data}, every colored dot represents a
\gi{} solution computed by \dpipa{} after color reduction performed by either
\randmd{}, \greedymd{}, \outliermd{}, or \metricmd{}. For each color reduction strategy, the
solution with the highest coverage (in the original, non-color-reduced graph) is indicated with a
rightward-pointing arrow, and the solution with minimum weight is indicated with a downward-facing arrow.

Because our color reduction schemes are designed to ensure good coverage, we are primarily interested in comparing the
coverage of solutions. The results indicate that, in general, our strategies outperform the baseline \randmd{} in terms of coverage, often by
large margins. In terms of coverage, the comparative performances of \greedymd{}, \metricmd{}, and \outliermd{} are somewhat difficult to disentangle.
Given that \greedymd{} is the simplest of the three, the most explainable, and the most generalizable (it does not require a metric embedding or
any parameter tuning), we favor it for future experiments. However, we leave as an interesting open direction to perform a more extensive
comparison of these methods on a larger data corpus.

We conclude this section with a discussion of the weight of solutions. We are not surprised that our strategies tend to produce higher-weight solutions
than \randmd{} since we are optimizing for coverage even if it means requiring that a solution walk visit outlier POIs. In particular, it is expected that
\greedymd{} performs poorly with respect to solution weight, since it explicitly favors outlier POIs. In applications
where the weight of the solution is of paramount interest, even at the expense of lowering coverage, the aforementioned
extended comparison of color reduction strategies may be of particular interest.

\section{Walk-Merging Algorithms}\label{sec:merge-proof}

In this section, we present preprocessing steps for the \prob{Minimum Spanning Eulerian Subgraph}
and the \greedymerge algorithm in detail.

\subsubsection*{Preprocessing for \prob{Minimum Spanning Eulerian Subgraph}.}
We say an edge set $\tilde{E} \subseteq E(G)$ is \emph{undeletable}
if there is an optimal solution for \prob{Minimum Spanning Eulerian Subgraph}
including all the edges in $\tilde{E}$.
We apply the following rules as preprocessing.

\begin{itemize}
  \item[Rule 1] If there is an edge with multiplicity at least $4$,
  then decrease its multiplicity by $2$.
  This makes the multiplicity of every edge either $1$, $2$ or $3$.
  \item[Rule 2] If there is an edge cut $e_1, e_2 \in E(G)$ of size $2$,
  then mark $e_1$ and $e_2$ as undeletable.
  For example, this includes (but is not limited to) the following:
  \begin{itemize}
    \item edges incident to a vertex with degree $2$.
    \item an edge with multiplicity $2$ that forms a bridge in the underlying simple graph of $G$.
  \end{itemize}
\end{itemize}

\subsubsection*{Algorithm \greedymerge.}
Consider the following heuristic for \prob{Minimum Spanning Eulerian Subgraph}.

\br

    \begin{algorithm}[H]
    \BlankLine
    Apply Rule 1 exhaustively. \;
    Construct a minimum spanning tree~$S$ of $G$ using a known algorithm. \;
    Greedily find a maximal cycle packing $\mathcal{C}$ in $G-S$ as follows:\;
	\begin{enumerate}
    	\item Let $U$ be the empty multigraph with $V(G)$
    	\item Iteratively add every edge $e \in E(G-S)$ to $U$ in order of nonincreasing weights.
    		If~$e$ creates a cycle $C$ in $U$, add $C$ to $\mathcal{C}$ and remove $C$ from $U$.\;
  	\end{enumerate}
  	\Return $G - \bigcup_{C \in \mathcal{C}}E(C)$
    \caption{\greedymerge}
    \end{algorithm}

This algorithm runs in $\Oh(m \log n)$ time.
In practice, we apply (part of) Rule~2 to find undeletable edges and include them in $S$ at step 2.
To prove the correctness of data reduction rules and algortihms,
we start by a simple, well-known observation on Eulerian graphs.

\begin{observation}\label{obs:eulerian-cycle}
  Let $C$ be a closed walk in an Eulerian multigraph $G$.
  If $G-C$ is connected, then $G-C$ is also Eulerian.
\end{observation}

\begin{proof}
  Removing a closed walk does not change the parity of the degree at each vertex.
  If all vertices in $G$ have even degrees, so do those in $G-C$, which implies
  that $G-C$ is Eulerian if it is connected.
  \qed
\end{proof}

We continue with the analysis of our two reduction rules.

\begin{lemma}
  Rule 1 is safe.
\end{lemma}

\begin{proof}
  By \Cref{lem:distinct-edges}, there must be an optimal solution with edge multiplicity at most $2$.
  From Observation \ref{obs:eulerian-cycle}, if there is an edge with multiplicity at least $4$,
  then removing $2$ of them (which can be seen as a closed walk) results in a connected Eulerian multigraph.
  \qed
\end{proof}

\begin{lemma}
  Rule 2 is safe.
\end{lemma}

\begin{proof}
  Let $S \subset V(G)$ be a nonempty vertex set such that $e_1, e_2$ separates $S$ from $V(G) \setminus S$.
  Then, any closed walk that visits $V(G)$ must contain at least one edge from $S$ to $V(G) \setminus S$
  and another from $V(G) \setminus S$ to $S$.
  Hence, both $e_1$ and $e_2$ must be in any solution.
  \qed
\end{proof}

We conclude this section by analyzing \greedymerge. 

\begin{theorem}
  \greedymerge correctly outputs a (possibly suboptimal) solution for \prob{Minimum Spanning Eulerian Subgraph}
  in~$\Oh(m \log n)$~time.
\end{theorem}

\begin{proof}
  We need to show that \greedymerge outputs a spanning Eulerian subgraph~$G'$
  (in this context, a \emph{subgraph} is also a multigraph) of $G$.
  From step~2, we know that $S$ is a spanning subgraph of $G$.
  From step~3, we have that $E(C) \subseteq E(G-S)$ for every $C \in \mathcal{C}$.
  From step~4, $G'$ is clearly a subgraph of $G$,
  and from $\bigcup_{C \in \mathcal{C}}E(C) \subseteq E(G-S)$,
  we have $E(G') \supseteq E(S)$, and hence $G'$ is spanning.
  Knowing that $G'$ is connected and from Observation \ref{obs:eulerian-cycle},
  a multigraph constructed by removing any closed walk remains Eluerian.
  Hence, $G'$ is a spanning Eluerian subgraph of $G$.

  We next analyze the running time.
  Rule 1 can be executed exhaustively in~$\Oh(m)$~time by iterating over all edges.
  At this point, $|E(G)| \leq 3 \cdot \binom{n}{2} \leq 2 n^2$.
  Now let us analyze the rest of the steps in \greedymerge.
  The running time of step $2$ is the same as that of Kruskal's algorithm, which is $\Oh(m \log m) \subseteq \Oh(m \log (2n^2)) = \Oh(m \log n)$.
  Similarly, step $3$ has the same running time as we iteratively examine edges and check for connectivity.
  Hence, the overall running time is in~$\Oh(m \log n)$.
  \qed
\end{proof}

\vfill
\pagebreak
\section{Experiment Details and Supplemental Results}\label{appendix:experiments}

The following table summarizes the test instances we used for our experiment.

\br
\noindent\begin{minipage}{\linewidth}
  \centering
  \makebox[\linewidth]{
      \centering

      \setlength{\tabcolsep}{0.5em} 
      {\renewcommand{\arraystretch}{1.3}

    \small
    \begin{tabular}{|c|c|r|r|r|r|}
        \hline
        \multicolumn{2}{|c|}{dataset} & \multicolumn{2}{c|}{\CRISP} & \multicolumn{2}{c|}{\DRONE} \\
        \hline
        \multicolumn{2}{|c|}{$n_\text{build}$} & 1,000 & 2,000 & 1,000 & 2,000 \\
        \hline
        \multicolumn{2}{|c|}{name} & \CRISPsmall & \CRISPbig & \DRONEsmall & \DRONEbig \\
        \hline
        \hline
        \multicolumn{2}{|c|}{number of vertices ($n$)} & 1,006 & 2,005 & 1,002 & 2,001 \\
        \multicolumn{2}{|c|}{number of edges ($m$)} & 18,695 & 41,506 & 19,832 & 44,089 \\
        \multicolumn{2}{|c|}{number of colors} & 4,200 & 4,200 & 3,204 & 3,254 \\
        \hline
        \multicolumn{2}{|c|}{\makecell{number of colors\\ at the starting vertex}} & 535 & 540 & 10 & 10 \\
        \hline
        \multirow{4}{*}{\makecell{number of colors\\ at a vertex}}
        & min & 0\hspace*{3ex} & 0\hspace*{3ex} & 0\hspace*{3ex} & 0\hspace*{3ex} \\
        & mean & 183.39 & 175.71 & 22.67 & 19.16 \\
        & max & 855\hspace*{3ex} & 876\hspace*{3ex} & 129\hspace*{3ex} & 129\hspace*{3ex} \\
        & stdev & 179.02 & 170.75 & 24.61 & 22.41 \\
        \hline
        \multirow{4}{*}{edge weight}
        & min & 0.000002 & 0.000000 & 0.51 & 0.43 \\
        & mean & 0.006971 & 0.005275 & 4.61 & 3.91 \\
        & max & 0.060926 & 0.060926 & 18.51 & 18.51 \\
        & stdev & 0.005354 & 0.004271 & 1.86 & 1.58 \\
        \hline
        \multicolumn{2}{|c|}{minimum spanning tree weight} & 1.109606 & 1.637212 & 1875.53 & 3118.65 \\
        \hline
        \multirow{2}{*}{diameter}
        & unweighted & 7 & 8 & 6 & 7 \\
        & weighted & 0.136846 & 0.138467 & 48.24 & 49.73 \\
        \hline
    \end{tabular}
      }
      }

  \captionsetup{hypcap=false}
  \captionof{table}{Corpus of test instances for \gi{}.}
  \label{tab:instances}
\end{minipage}

\subsubsection*{Experiment Environment.}
We implemented our code with C++ (using C++17 standard).
We ran all experiments on identical hardware,
equipped with 80 CPUs (Intel(R) Xeon(R) Gold 6230 CPU @ 2.10GHz)
and 191000 MB of memory, and running Rocky Linux release 8.8.
We used Gurobi Optimizer 9.0.3 as the ILP solver, parallelized over CPUs.

\begin{figure}[p]
  \begin{subfigure}[t]{\linewidth}
    \includegraphics[width=0.9\linewidth]{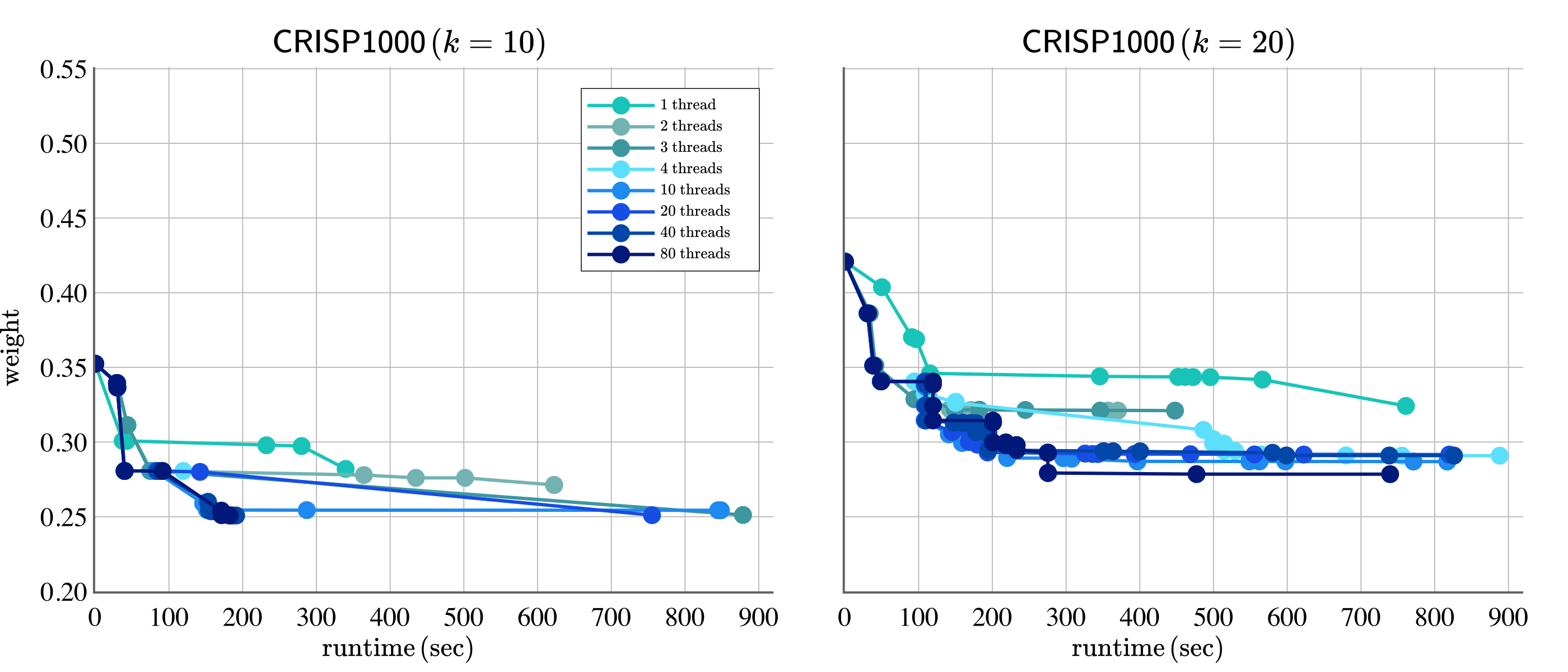}
  \end{subfigure}
  %
  %
  \begin{subfigure}[t]{\linewidth}
    \includegraphics[width=0.9\linewidth]{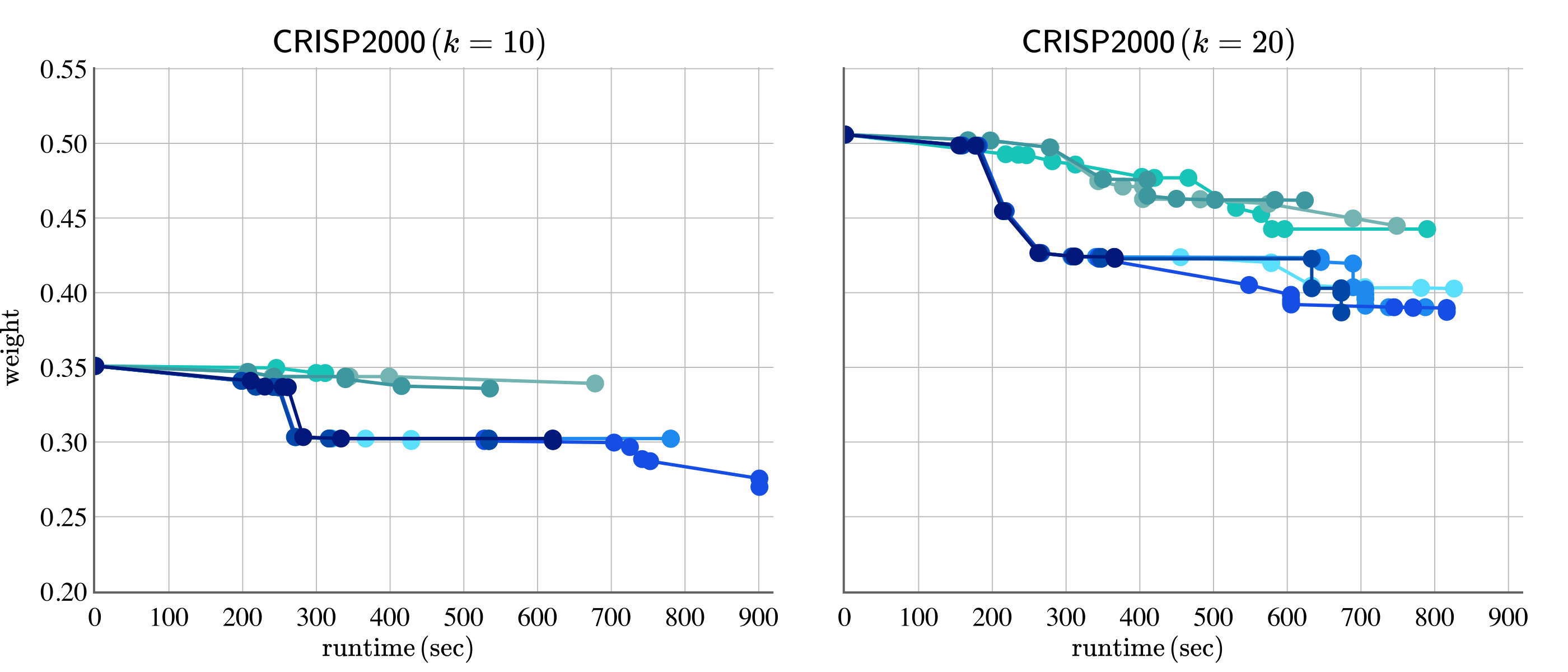}
  \end{subfigure}
  %
  %
  \begin{subfigure}[t]{\linewidth}
    \includegraphics[width=0.9\linewidth]{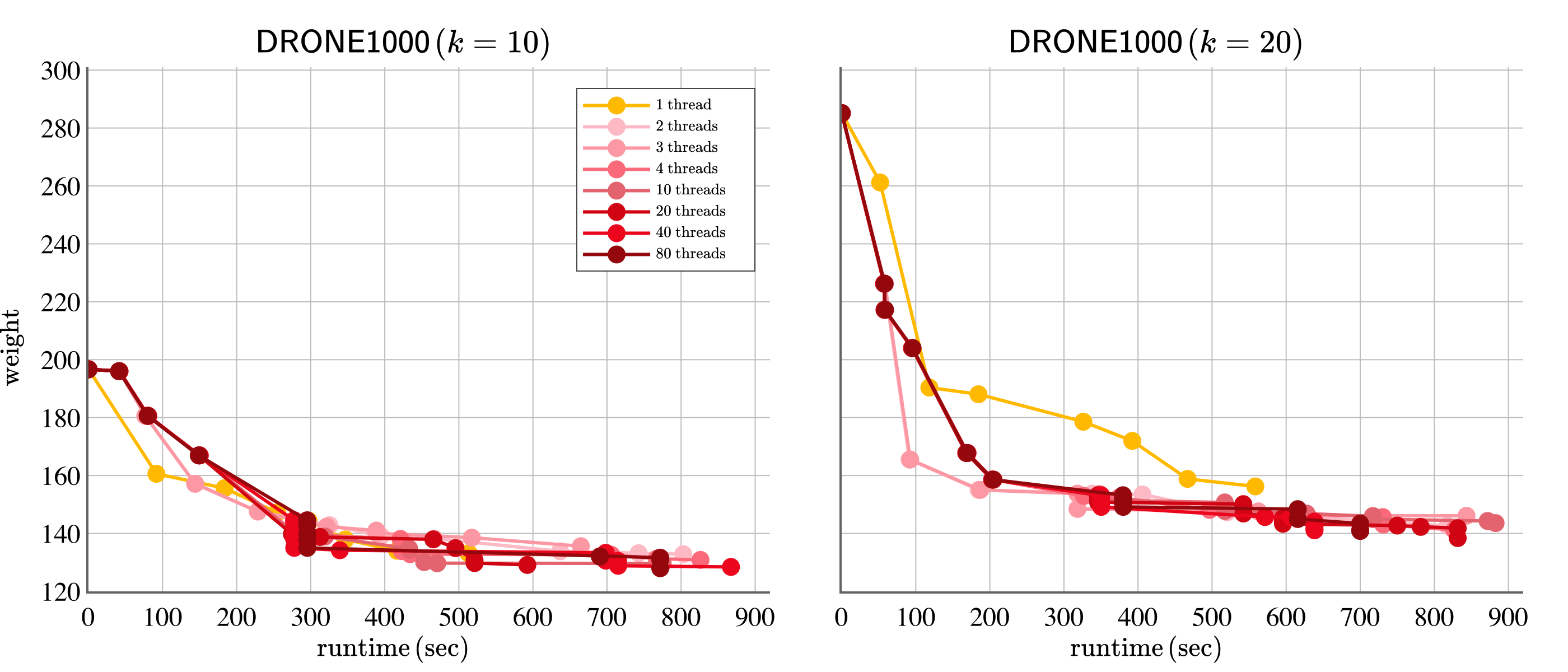}
  \end{subfigure}
  %
  %
  \begin{subfigure}[t]{\linewidth}
    \includegraphics[width=0.9\linewidth]{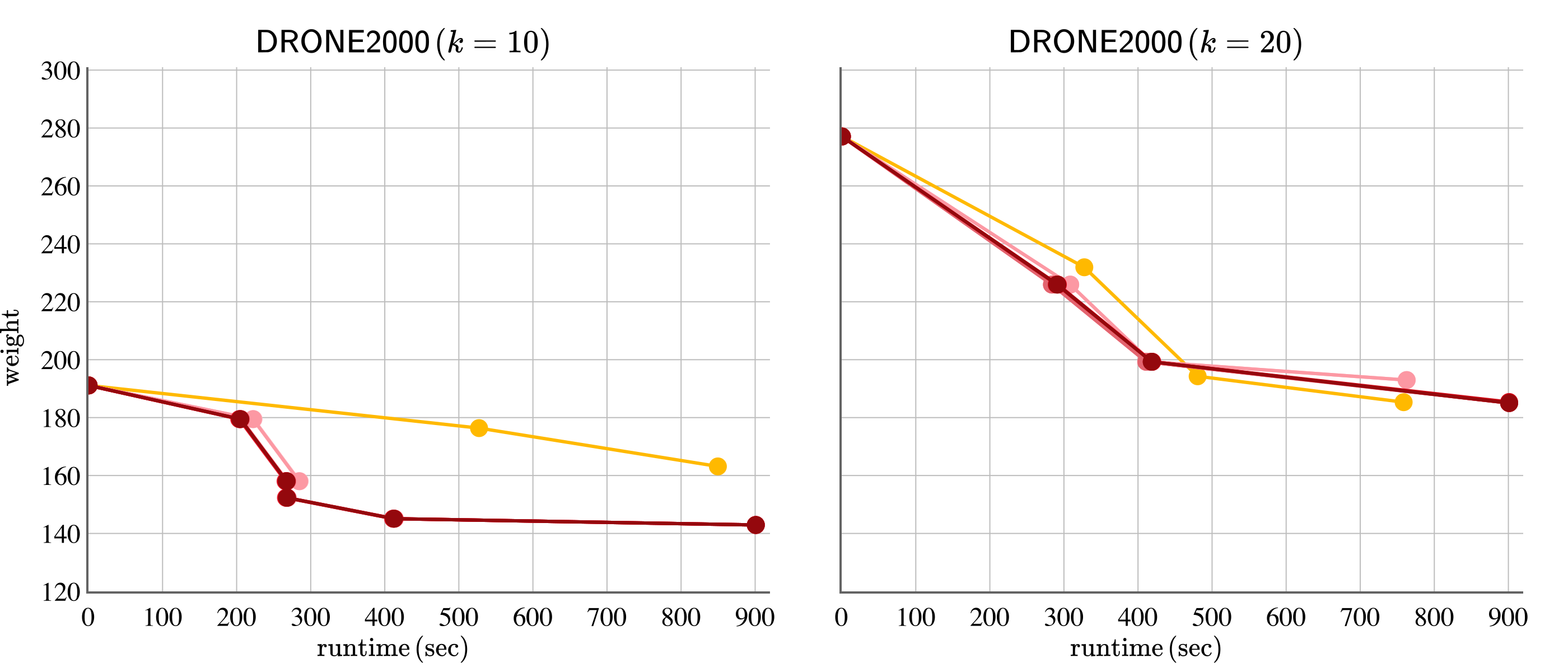}
  \end{subfigure}
  \caption{\label{fig:ilp-multithreading} Results of multithreading experiments
  for \ilpipa{}, executed on datasets \CRISP{} (top two rows) and \DRONE{}
  (bottom two rows), with $k = 10$ (left column) or $k = 20$ (right column).}
\end{figure}

\subsection{Multithreading Analysis}\label{appendix:multithreading}

Here, we present results of our multithreading experiments for \ilpipa{} (visualized in ~\Cref{fig:ilp-multithreading}). We configured the Gurobi ILP solver to output each feasible solution as soon as it is found, so
to understand the impact of multithreading we are interested in determining the lowest-weight
feasible solution identified in a given time limit, for various thread counts. Predictably, the
general trend is clear: multi-threaded implementations provide lower-weight solutions faster than
single threaded solutions.

\vfill
\pagebreak 

\subsection{Additional Empirical Results \& Figures}\label{appendix:additional}

Based on our findings in \Cref{appendix:color-reduction}, we also compared all three inspection planning algorithms when \dpipa{} is paired with \metricmd{} color reduction (instead of \greedymd{}, as shown in \Cref{fig:headline}). The results are shown in \Cref{fig:headline2}. 

\begin{figure}[h!]
  \vspace{-1em}
  \begin{subfigure}{\textwidth}
    \includegraphics[width=\linewidth]{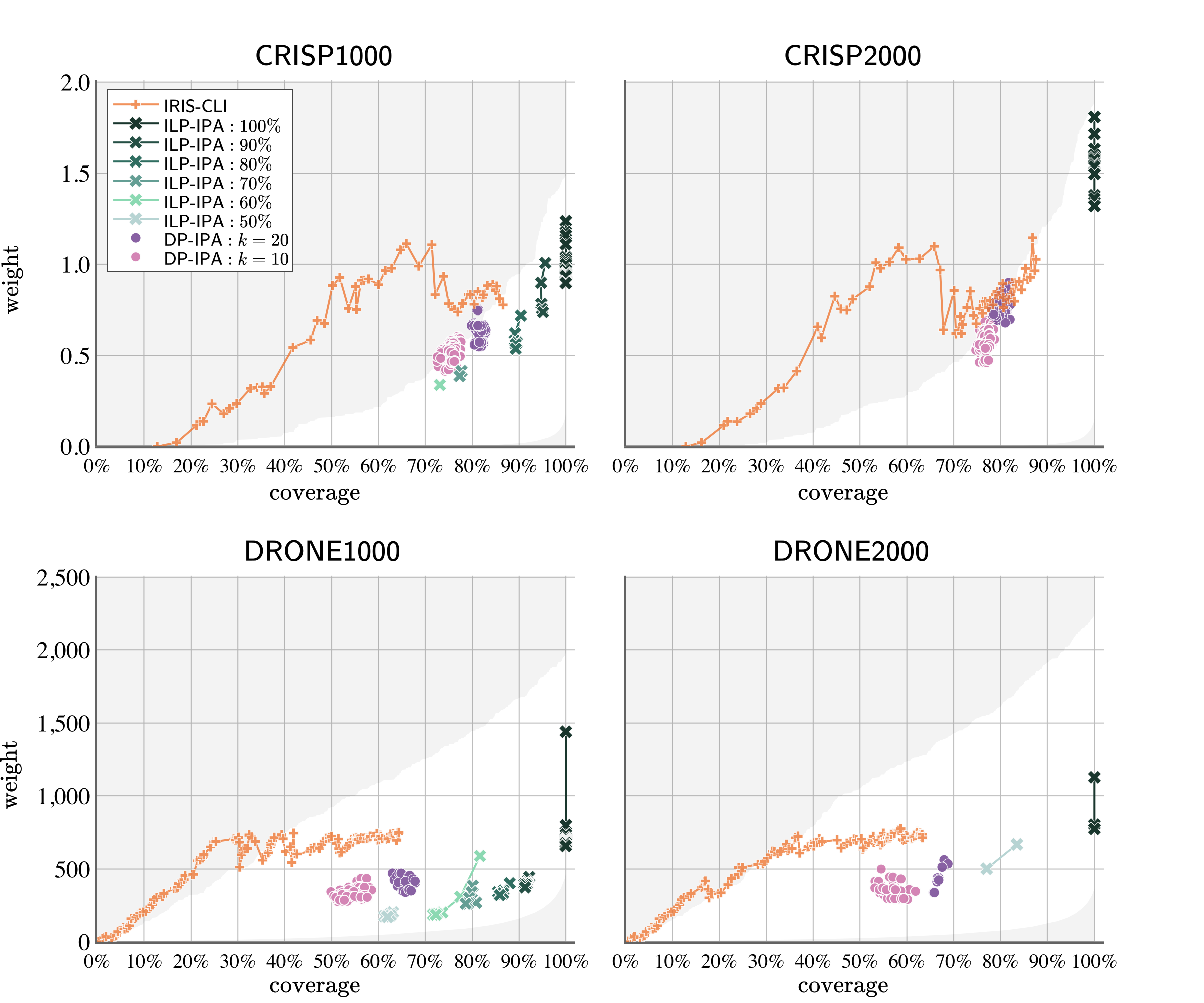}
  \end{subfigure}
  \caption{\label{fig:headline2}  Performance of \iriscli{}, \ilpipa{} and
  \dpipa{} (with \metricmd) on \DRONE~and \CRISP~benchmarks. Each data point represents a computed inspection plan; coverage is shown as a percentage of all POIs in the input graph. The area shaded in gray is outside the upper/lower bounds given in Section~\ref{sec:theory-bounds}.}
\vspace{-2em}
\end{figure}

They are qualitatively quite similar, but we observe that this approach has the disadvantage of requiring the POI similarity function to be a metric. 
\vfill
\pagebreak
We showed the results of applying our color partitioning methods in combination with \greedymd{} and \metricmd{} for the 1000-node graphs in \Cref{fig:color-reduction-data}; the analogous results for the 2000-node graphs are included below in \Cref{fig:color-part-appendix}.
\begin{figure}[h!]
  \begin{subfigure}{\textwidth}
    \includegraphics[width=\linewidth]{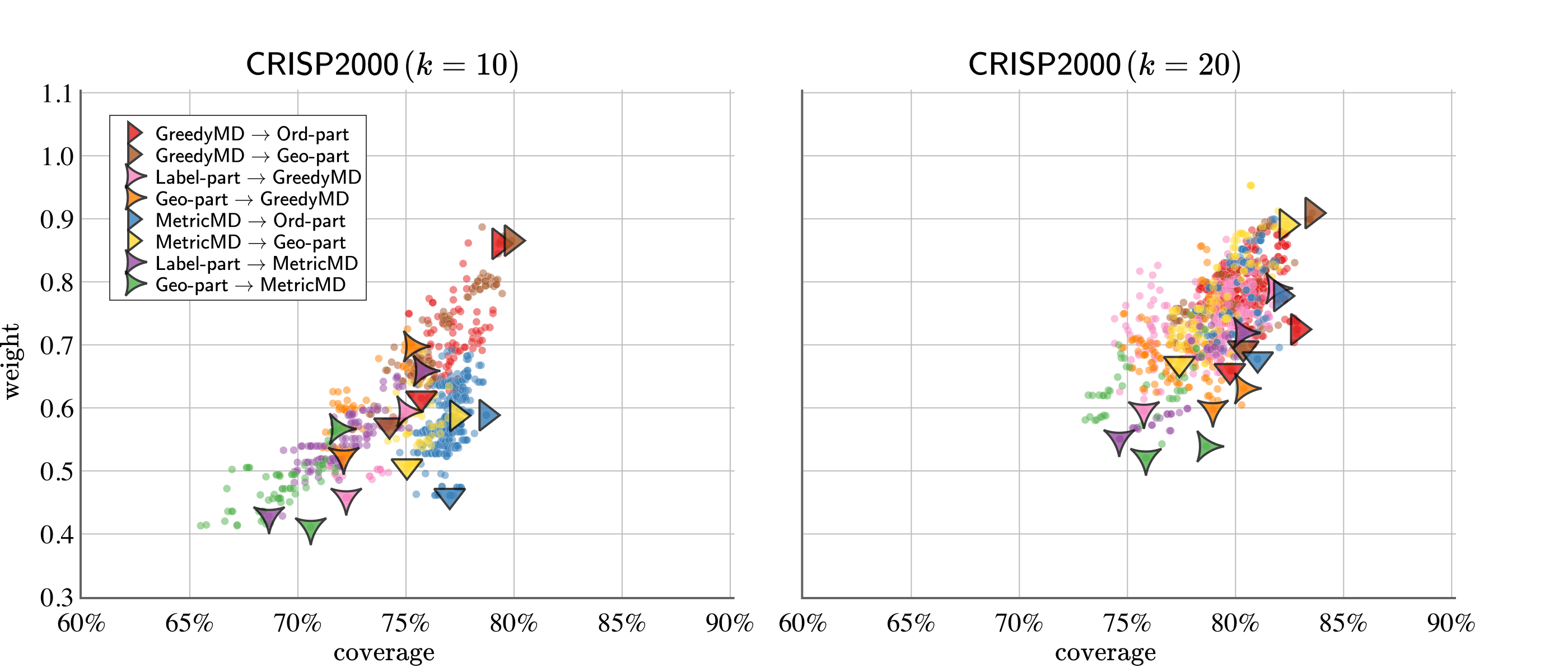}
  \end{subfigure}
  \begin{subfigure}{\textwidth}
    \includegraphics[width=\linewidth]{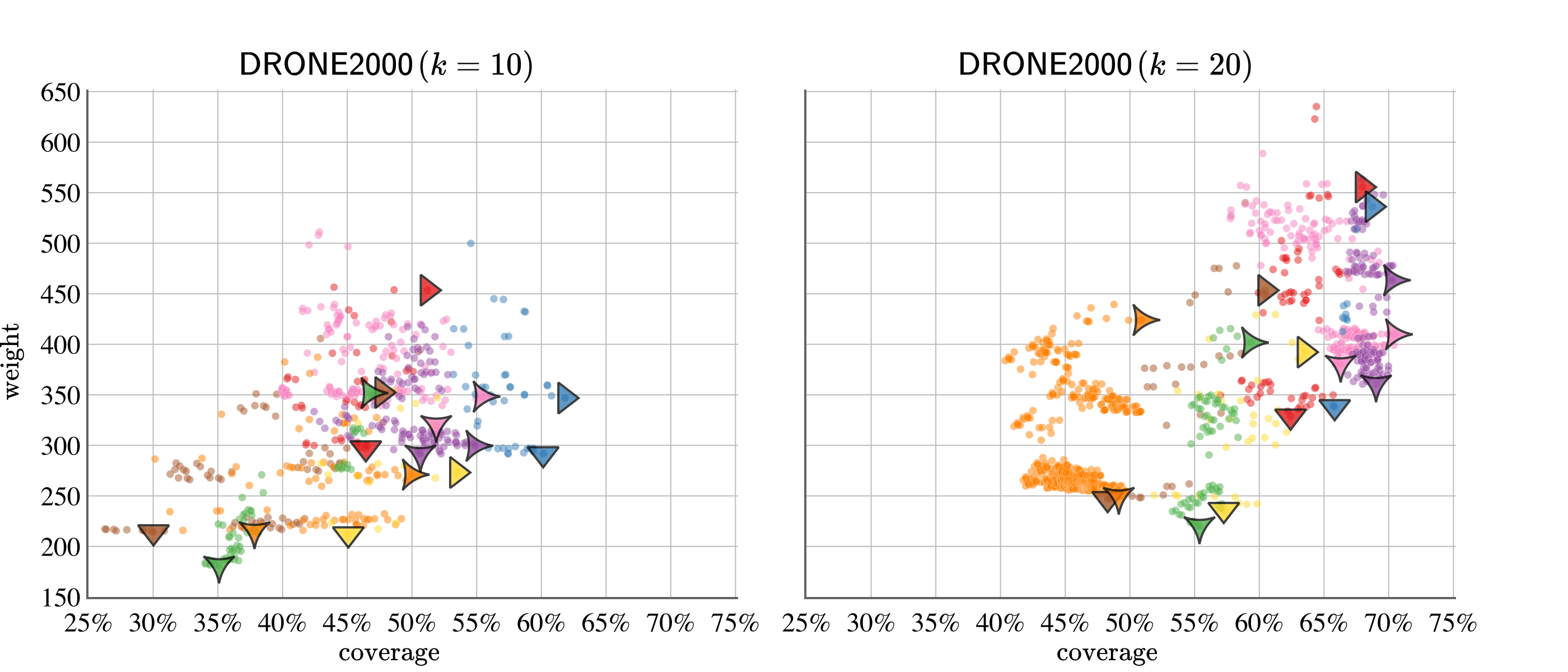}
  \end{subfigure}
  \caption{\label{fig:color-part-appendix}
  Results of our color partitioning experiments for datasets \CRISPbig{} and \DRONEbig{} with $k \in \{10, 20\}$. Every data point represents a solution computed using \dpipa{} and \exactmerge{}. For each combination of color reduction and color partitioning strategies, the solution with maximum coverage is indicated with a rightward-pointing arrow, and the solution with minimum weight is indicated with a downward-pointing arrow.}
\end{figure}

\vfill
\pagebreak
After applying color reduction and partitioning, we have a set of walks that need merging, as discussed in \Cref{sec:walk-merging} with additional details in \Cref{sec:merge-proof}. The empirical results of comparing these approaches are shown in \Cref{fig:merge-experiments}.

\begin{figure}[h!]
  \begin{subfigure}{\textwidth}
    \includegraphics[width=\linewidth]{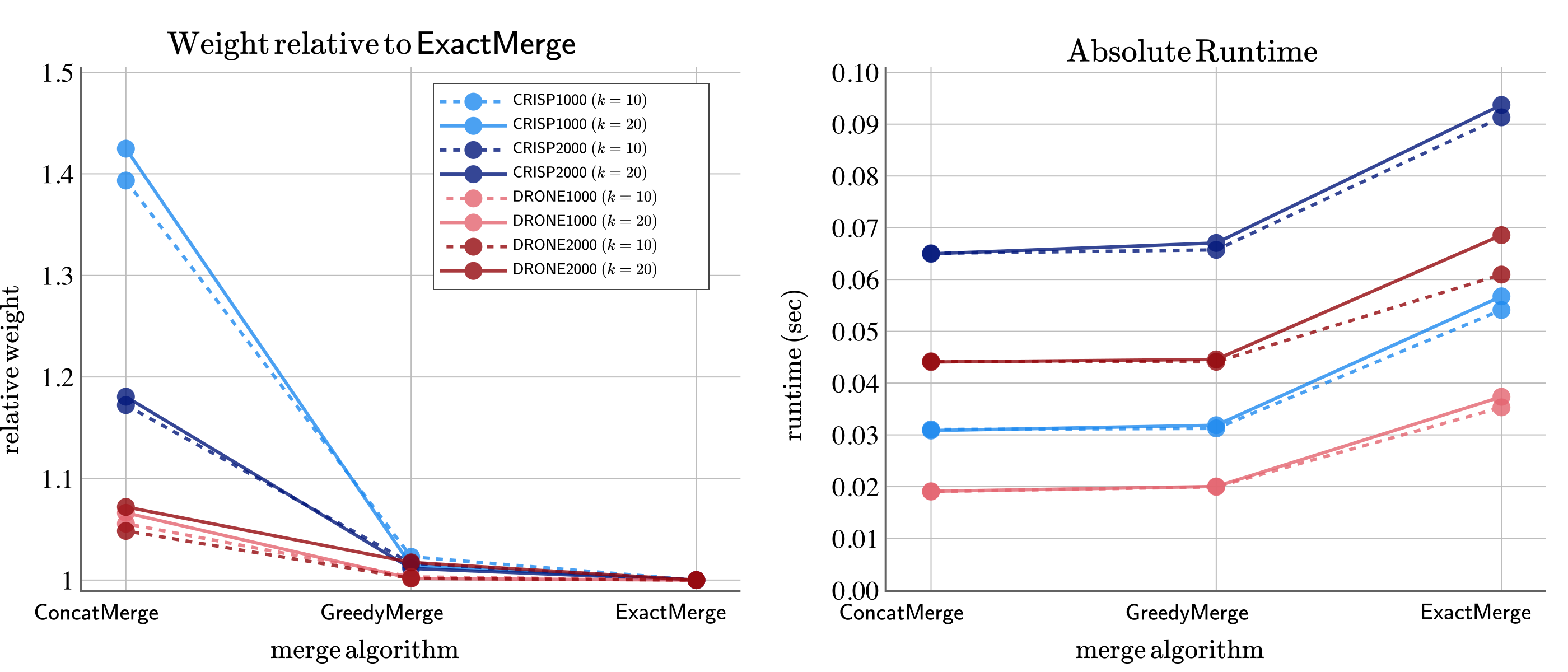}
  \end{subfigure}
  \caption{\label{fig:merge-experiments} Experimental results comparing \concatmerge, \greedymerge, and \exactmerge.
  On the left, the weight of the combined walk is compared (relative to \exactmerge). On the right, runtimes are compared.
  In all experiments, three walks were combined.}
\end{figure}

\Cref{tab:ilp-vs-dp} compares the solution weights generated by \dpipa{} and \ilpipa{}.
In this table, both \dpipa{} and \ilpipa{} are used to produce solutions on color-reduced graphs.
The table shows solution weights by \ilpipa{} relative to the optimal weights
that \dpipa{} produces.
The trend is clear: \ilpipa{} becomes less competitive with \dpipa{} as $n$ ($n_\text{build}$) and $k$ grow.

\begin{table}[h]
  \setlength{\tabcolsep}{8pt}
    \centering
    \small
    \begin{tabular}{l l l l l l l l}
      \multicolumn{2}{c}{} & \multicolumn{3}{c}{\CRISP{}} & \multicolumn{3}{c}{\DRONE{}} \\
      \cmidrule(lr){3-5} \cmidrule(lr){6-8}
      $n_\text{build}$ &$k$ &mean &min &max &mean &min &max \\
      \hline
      \multirow{2}*{$1,000$} &$10$ &1.03 &1.00 &1.16   &1.04 &1.00 &1.13\\
                            &$20$ &1.04 &1.00 &1.31   &1.05 &1.00 &1.18\\[3pt]
      \multirow{2}*{$2,000$} &$10$ &1.12 &1.00 &1.38   &1.17 &1.03 &1.38\\
                            &$20$ &1.16 &1.01 &1.34   &1.19 &1.02 &1.51\\
    \end{tabular}
    \vspace{1pt}
    \caption{\label{tab:ilp-vs-dp} Comparison of solution weights generated by \ilpipa{}, relative to the lowest weight produced by \dpipa{}; both solvers had a 15-minute timeout.}
    \vspace{-2.5em}
  \end{table}

\end{document}